\documentclass{article}

\usepackage[main, final]{neurips_2026}

\usepackage[utf8]{inputenc}
\usepackage[T1]{fontenc}
\usepackage{hyperref}
\usepackage{url}
\usepackage{booktabs}
\usepackage{amsmath}
\usepackage{mathtools}
\usepackage{amssymb}
\usepackage{amsfonts}
\usepackage{amsthm}
\usepackage{nicefrac}
\usepackage{microtype}
\usepackage{graphicx}
\usepackage{subcaption}
\usepackage{placeins}
\usepackage{float}
\usepackage{algorithm}
\usepackage[noend]{algpseudocode}
\usepackage[table]{xcolor}
\usepackage[capitalize,noabbrev]{cleveref}

\setcitestyle{authoryear,round}
\let\cite\citep

\newcommand{\defeq}{\vcentcolon=}

\definecolor{Best}{RGB}{230,245,208}
\definecolor{Second}{RGB}{198,219,239}
\definecolor{Third}{RGB}{241,238,246}

\newcommand{\best}[1]{\cellcolor{Best}\textbf{#1}}
\newcommand{\second}[1]{\cellcolor{Second}\textbf{#1}}
\newcommand{\third}[1]{\cellcolor{Third}\textbf{#1}}

\newcommand\ie{\textit{i.e.}}
\newcommand\eg{\textit{e.g.}}
\newcommand{\RomanNumeralCaps}[1]{\MakeUppercase{\romannumeral #1}}

\theoremstyle{plain}
\newtheorem{theorem}{Theorem}[section]
\newtheorem{proposition}[theorem]{Proposition}
\newtheorem{lemma}[theorem]{Lemma}
\newtheorem{corollary}[theorem]{Corollary}
\theoremstyle{definition}
\newtheorem{definition}[theorem]{Definition}
\newtheorem{assumption}[theorem]{Assumption}
\theoremstyle{remark}
\newtheorem{remark}[theorem]{Remark}

\title{Mitigating Global Feature Collapse in Graph ODEs \\ via Hysteretic Topology Evolution}

\author{%
Qinhan Hou\\
University of Helsinki\\
\texttt{hou.qinhan@helsinki.fi}\\
\And
Jing Tang\\
University of Helsinki\\
\texttt{jing.tang@helsinki.fi}
}

\begin{document}

\maketitle

\begin{abstract}
    Graph neural ordinary differential equations provide continuous-time propagation on graphs, but their long-time behavior is largely determined by the asymptotic structure of the mixing operator. We show that diffusion-style Graph ODEs with positive irreducible mixing face a \emph{monostability trap}, where node features converge to a rank-one consensus subspace. This reveals a structural limitation of positive diffusion-style continuous graph propagation: within this consensus class, exact long-time avoidance requires changing the effective support; otherwise, smooth suppression should be interpreted as finite-horizon or metastable mitigation. We address this by introducing a \emph{hysteretic mechanism} for topology evolution, where candidate edges carry latent bistable potentials in a double-well landscape. A node-feature-conditioned force field evolves these potentials, and a smooth sigmoid gate converts them into differentiable propagation weights. We prove global consensus under the stated mixing assumptions and block-wise consensus under idealised support separation. For consensus-form dynamics, we further establish a finite-horizon perturbation bound under small cross-block leakage, and empirically evaluate the resulting nonlinear HGODE on node- and graph-level benchmarks.
\end{abstract}

\section{Introduction}

Inspired by neural ordinary differential equations~\cite{chen2018neural}, Graph
ODEs (GODEs)~\cite{poli2019graph, liu2025graph} evolve node representations
continuously over a graph, supporting adaptive long-range propagation beyond
finite-depth message passing. However, continuous-time propagation also
changes the nature of feature collapse. In finite-depth Graph Neural Networks
(GNNs), over-smoothing is usually described as the effect of increasing
aggregation layers. In diffusion-style GODEs, the corresponding risk is more
structural: when the mixing operator remains positive and irreducible,
consensus is the long-time attractor of the underlying dynamical system.

Though the feature collapse can be described as a classical consensus
phenomenon, its implication for Graph ODE design is important: within positive
diffusion-style propagation, avoiding exact long-time collapse requires
controlling the asymptotic mixing support, rather than merely rescaling
positive weights or regularizing feature dynamics. More specifically, the
effective propagation operator must be able to move away from a single
irreducible support toward separated or weakly interacting components. We
therefore introduce a hysteretic mechanism for topology evolution and design an
ODE framework, \emph{Hysteresis Graph ODE} (HGODE), to mitigate this global collapse.
In our proposed framework, each candidate edge is controlled by a latent
bistable potential in a double-well landscape and can occupy either a connected
or insulated phase, a process we call \emph{topology polarization}. A learned
force field computed from node representations drives these potentials,
yielding a coupled system in which topology and features co-evolve. Unlike
ACMP/GREAD-style reaction or phase-transition dynamics in the node-feature
state~\cite{wang2022acmp, choi2023gread}, HGODE places the bistable state on
candidate edges, so the phase transition changes the effective propagation
weights and, under the analysis projection, the induced support. Empirically,
we use soft-attention-based GODE diagnostics to demonstrate the predicted
feature collapse by tuning the attention temperature, and we evaluate the
resulting framework on real-world benchmarks across graph
types and tasks.

Overall, our contributions are:

\begin{itemize}
    \item We identify the \emph{monostability trap} as a structural limitation of
          diffusion-style GODEs: positive irreducible mixing drives the long-time
          dynamics to rank-one consensus. We frame this classical consensus mechanism as
          a failure mode of dense positive continuous graph propagation. Additionally, we
          derive the corresponding design principle for this type of propagation: exact
          asymptotic avoidance requires the effective support to become reducible.
    \item We instantiate this principle through Hysteresis Graph ODE (HGODE), a coupled
          feature--topology ODE in which double-well edge potentials, learned forces, and
          smooth sigmoid gates induce connected/insulated edge phase transitions and
          persistent topology polarization while preserving differentiability; its strictly
          positive gates give finite-horizon (metastable) mitigation, which we separate from
          the idealized block-wise limit.
    \item We validate the mechanism with theory-driven synthetic diagnostics and
          real-world node- and graph-level benchmarks, including ablations for
          hysteresis, topology search, and the force-margin objective.
\end{itemize}

\section{Related Work}

\textbf{Graph neural ODEs and continuous-depth GNNs.}
Graph neural ODEs extend GCN-style propagation~\cite{kipf2017semisupervised, chen2020simple, luo2025can}
by viewing message passing as continuous-time dynamics on graphs. Early continuous-depth models
such as GDE~\cite{poli2019graph}, CGNN~\cite{pmlr-v119-xhonneux20a}, and
GRAND~\cite{chamberlain2021grand} instantiate this idea mainly through diffusion-style feature
flows, which provide an infinite-depth view of graph propagation. Subsequent work has modified the
feature dynamics or the representation geometry to mitigate over-smoothing and improve expressivity:
Neural Sheaf Diffusion~\cite{bodnar2022neural} uses sheaf-valued diffusion; FLODE~\cite{maskey2024fractional},
FROND~\cite{kang2024frond}, and DRAGON~\cite{zhao2024dragon} alter the temporal or spectral
law through fractional-order dynamics; GREAD~\cite{choi2023gread} and
ACMP~\cite{wang2022acmp} introduce reaction-diffusion or phase-transition dynamics in the node-feature
state; and BuNN~\cite{bamberger2025bundle} performs diffusion over bundle-valued representations.
Recent variants further introduce control, second-order, or equivariant dynamics, as in
CSG-ODE~\cite{pmlr-v267-wang25dd}, SEGNO~\cite{liu2024segno}, and
DuSEGO~\cite{wang2025dusegodualsecondorderequivariant}. These methods are complementary to
HGODE, since they mainly change how features evolve on a given or softly reweighted propagation operator, whereas HGODE makes edge connectivity itself a continuous latent state and studies how its polarization affects the long-time mixing structure.

\textbf{Over-smoothing and long-time collapse.}
Over-smoothing in discrete and continuous GNNs is commonly explained through repeated aggregation,
Laplacian contraction, or random-walk convergence toward low-frequency or consensus
subspaces~\cite{oono2020graph, rusch2023survey, wu2023demystifying, wu2023a, hou2025exploring};
\citet{roth2024rank} identify aggregation-driven rank collapse as a common cause of over-smoothing
and over-correlation in discrete-depth GNNs.
Residual connections, normalization methods such as PairNorm~\cite{Zhao2020PairNorm}, and
personalized diffusion or scattering methods~\cite{gasteiger2018predict, min2020scattering, koishekenov2023reducing}
are effective finite-depth remedies, and source-term diffusion such as GRAND++~\cite{thorpe2022grandpp}
changes the equilibrium equation so that representations need not become node-constant; none of these
directly controls the asymptotic support of the mixing operator. Our analysis isolates a complementary long-time failure mode: an irreducible row-stochastic operator
has a single consensus attractor, so exact asymptotic avoidance requires a reducible effective support,
while smooth suppression of positive weights gives finite-horizon or metastable mitigation; HGODE
introduces persistent edge-level phases that suppress harmful mixing over the learned horizon.

\textbf{Graph structure learning and adaptive topology.}
Graph structure learning refines the input graph through adaptive adjacency matrices, metric graphs,
denoising objectives, diffusion preprocessing, attention weights, or geometric rewiring
criteria~\cite{dgcnn, zheng2024rethinking, gasteiger2019diffusion, 10.1145/3394486.3403049, nguyen2023borf}. GAT-style attention
mechanisms~\cite{velivckovic2018graph, lee2019attention, ye2021sparse} can also be interpreted as
feature-dependent soft structure learning. However, dense positive attention weights may preserve
global irreducible mixing in the long-time limit, and many structure-learning methods learn a static
or instantaneous adjacency rather than a persistent topological state. HGODE differs by assigning
each candidate edge a bistable latent potential that co-evolves with node features, so that the
topology itself carries memory and persistent connected/insulated phases. This positioning makes
LDS/IDGL-style~\cite{pmlr-v97-franceschi19a, 10.5555/3495724.3497344} structure learning, robust graph
denoising, and sparse/top-$k$ attention natural complementary comparisons; our main experiments
focus on the continuous-depth diffusion baselines targeted by the theory.

\section{Monostability Trap and Hysteretic Topology Dynamics}
\label{sec:monostable_trap}
\subsection{Notations}
\label{notations}
For a given graph $\mathcal{G}=\{\mathcal{V},\mathcal{E}\}$, we represent $\mathcal{V}=\{v_1, v_2, \dots, v_N\}$ and $\mathcal{E}=\{e_{ij} \mid v_i, v_j \in \mathcal{V}\}$ as the node set and edge set, respectively. For each node $v_i$, we define an $m$-dimensional feature vector $\mathbf{h}_i \in \mathbb{R}^m$, forming the global feature matrix $\mathbf{H} \in \mathbb{R}^{N\times m}$. Throughout, $t\ge0$ is continuous depth (ODE time), $T$ the integration horizon, and $\mathbf{H}(t)$ the feature trajectory from $\mathbf{H}(0)$; an unindexed $\mathbf{H}$ is a generic feature matrix. For $k\in\{1,\dots,m\}$, $\mathbf{H}_{\cdot k}(t)\defeq(\mathbf{H}_{1k}(t),\dots,\mathbf{H}_{Nk}(t))^\top$ collects the $k$-th feature coordinate over all nodes, and $\mathrm{diam}(\mathbf{x})\defeq\max_i x_i-\min_i x_i$ for $\mathbf{x}\in\mathbb{R}^N$.

Since our framework allows for asymmetric interactions, we define the degree
matrix $\mathbf{D}(t)=\mathrm{diag}(d_1(t),\dots,d_N(t))$ with $d_i(t)=\sum_j
    \widetilde{\mathbf{A}}_{ij}(t)$, where
$\widetilde{\mathbf{A}}(t)=\mathbf{A}(t)+\epsilon \mathbf{I}$ adds a small
self-loop for numerical stability. The diffusion is governed by the
row-stochastic matrix
$\mathbf{P}(t)=\mathbf{D}(t)^{-1}\widetilde{\mathbf{A}}(t)$; when the time
dependence is irrelevant we write $\mathbf{P}=\mathbf{D}^{-1}\widetilde{\mathbf{A}}$,
and the consensus statements below hold for any row-stochastic $\mathbf{P}$.
We denote the topological support of $\mathbf{P}$ as
$\mathcal{E}_{supp} \defeq \{(i,j): \mathbf{P}_{ij} > 0\}$.

For the proposed framework, we introduce a continuous latent potential matrix
$\mathbf{U} \in \mathbb{R}^{N \times N}$. To ensure scalability, we restrict
computations to a sparse candidate set $\mathcal{E}_{cand} \subset
    \mathcal{V}\times\mathcal{V}$ that is fixed during each ODE solve (observed
edges plus 2-hop, Laplacian-eigenvector, and random proposals;
Sec.~\ref{sec:active_set}). A smooth sigmoid gate maps latent
potentials to effective candidate-edge weights, so the forward HGODE dynamics
remain differentiable. We denote the topological force function as
$\mathcal{F}: \mathbb{R}^m \times \mathbb{R}^m \to \mathbb{R}$ and the force on
the directed edge $e_{ij}$ as
$\mathcal{F}_{ij}(t)\defeq\mathcal{F}(\mathbf{h}_i(t),\mathbf{h}_j(t))$, which the
theory only needs to be bounded and piecewise continuous in $t$; $\mathcal{F}_\theta$
is its learned version (Sec.~\ref{sec:force_margin}). Finally,
$\tau_{feat}$ and $\tau_{topo}$ denote the timescale constants for feature and
topological evolution, respectively. When a hard support is discussed below, it
is a thresholded analysis/evaluation projection of the learned smooth weights,
not part of the ODE vector field.

All of the proofs can be found in the Appendix~\ref{sec:proofs}.

\subsection{The Monostability of Directed Consensus}

We analyze the asymptotic behavior of continuous-time GNNs. Unlike previous
works that rely on undirected Dirichlet energy, we consider the general case of
directed, non-symmetric interactions characteristic of global soft-attention
mechanisms, using \textit{consensus dynamics} on positively weighted graphs.

\begin{definition}[Directed Diffusion Dynamics]
    \label{def:directed_diffusion}
    Let $\mathbf{P} = \mathbf{D}^{-1}\widetilde{\mathbf{A}}$ be a (time-invariant) row-stochastic
    transition matrix of the graph.
    The evolution of node features $\mathbf{H}(t)$ is governed by the flow
    \begin{equation}
        \label{eq:directed_diffusion}
        \frac{d\mathbf{H}(t)}{dt} = -( \mathbf{I} - \mathbf{P} ) \mathbf{H}(t),
    \end{equation}
    \ie, $\dot{\mathbf{h}}_i(t)=\bar{\mathbf{h}}_i(t)-\mathbf{h}_i(t)$ with
    $\bar{\mathbf{h}}_i(t)\defeq\sum_j\mathbf{P}_{ij}\mathbf{h}_j(t)$: each node moves
    toward its current row-stochastic neighbor average.
\end{definition}

We begin with the time-invariant case and extend it to time-varying,
attention-like operators in Sec.~\ref{sec:time_varying_main}.

\begin{assumption}[Strongly connected positive support]
    \label{ass:positivity}
    The directed support graph $(\mathcal{V},\mathcal{E}_{supp})$ of $\mathbf{P}$, with
    $\mathcal{E}_{supp}=\{(i,j):\mathbf{P}_{ij}>0\}$, is strongly connected; equivalently,
    the row-stochastic matrix $\mathbf{P}$ is irreducible.
\end{assumption}

This assumption is satisfied by dense soft-weight constructions (\eg, global
attention or dense similarity kernels) and by any fixed support whose positive
weights keep the directed graph irreducible. Explicit masking or sparsification
can break irreducibility, so such operators are outside this theorem rather
than counterexamples to it. Under Assumption~\ref{ass:positivity}, we invoke
the Perron--Frobenius theorem to characterize the resulting consensus behavior,
using the following notion of collapse.

\begin{definition}[Global and block-wise consensus collapse]
    \label{def:collapse}
    Let $\mathcal{C}_{\mathrm{global}}\defeq\{\mathbf{1}\mathbf{y}^\top:\mathbf{y}\in\mathbb{R}^m\}$.
    A trajectory $\mathbf{H}(t)$ undergoes \emph{global consensus collapse} if
    $\mathrm{dist}_F(\mathbf{H}(t),\mathcal{C}_{\mathrm{global}})\defeq\min_{\mathbf{y}\in\mathbb{R}^m}\|\mathbf{H}(t)-\mathbf{1}\mathbf{y}^\top\|_F\to0$
    as $t\to\infty$. For a node partition $\{\mathcal{C}_1,\dots,\mathcal{C}_K\}$, it reaches
    \emph{block-wise consensus} if $\mathbf{H}_{\mathcal{C}_k}(t)\to\mathbf{1}\mathbf{y}_k^\top$ for every $k$.
\end{definition}

Block-wise consensus with distinct $\mathbf{y}_k$ is not global collapse, although
features still agree within blocks; \emph{mitigating feature collapse} means
avoiding or delaying global consensus collapse.

\begin{theorem}[Consensus Trap under Irreducible Positive Mixing]
    \label{thm:monostability}
    Under Assumption~\ref{ass:positivity}, the Markov matrix $\mathbf{P}$ admits a unique stationary distribution $\boldsymbol{\pi}\succ 0$ such that $\boldsymbol{\pi}^\top \mathbf{P}=\boldsymbol{\pi}^\top$ and $\boldsymbol{\pi}^\top \mathbf{1}=1$.
    Moreover, every solution of the directed consensus flow in
    Eq.~\eqref{eq:directed_diffusion} converges to the rank-one consensus subspace:
    \begin{equation}
        \lim_{t\to\infty}\mathbf{H}(t)=\mathbf{1}\big(\boldsymbol{\pi}^\top \mathbf{H}(0)\big).
    \end{equation}
    Let \(\Delta \defeq \min_{i\ge 2}\operatorname{Re}(1-\lambda_i(\mathbf{P}))>0\). Then
    \(\|\mathbf{H}(t)-\mathbf{1}\boldsymbol{\pi}^\top\mathbf{H}(0)\|=O(e^{-(\Delta-\varepsilon)t})\)
    for every \(\varepsilon>0\), with rate exactly \(\Delta\) when \(\mathbf{P}\) is
    diagonalizable (Appendix~\ref{app:jordan} gives the Jordan-form refinement).
\end{theorem}

Theorem~\ref{thm:monostability} shows that, within this positive irreducible
regime, every trajectory undergoes global consensus collapse
(Definition~\ref{def:collapse}). Escaping this trap therefore requires changing the
asymptotic mixing structure, most directly by making the effective propagation
reducible (\eg, block-diagonal up to permutation), which yields multiple invariant
subspaces corresponding to different clusters.

\begin{remark}[Invariant consensus subspace]
    \label{rem:invariant_subspace}
    The zero modes of Eq.~\eqref{eq:directed_diffusion} form $\ker(\mathbf{I}-\mathbf{P})$: $\mathrm{span}\{\mathbf{1}\}$
    for irreducible $\mathbf{P}$, and $\mathrm{span}\{\mathbf{1}_{\mathcal{C}_1},\dots,\mathbf{1}_{\mathcal{C}_K}\}$ for
    $\mathbf{P}=\mathrm{diag}(\mathbf{P}_1,\dots,\mathbf{P}_K)$ with irreducible blocks. Support-preserving reweighting,
    rate rescaling, lazy mixing, and the Euler layer $\mathbf{I}-h(\mathbf{I}-\mathbf{P})$ leave it unchanged, so
    discrete-depth propagation shares this subspace~\cite{roth2024rank}; $K>1$ separated blocks change it, and
    $\mathcal{C}_{\mathrm{global}}$ then stays invariant but is no longer attractive. Source terms such as
    GRAND++~\cite{thorpe2022grandpp} act differently, by changing the equilibrium equation
    (Proposition~\ref{prop:invariant_subspace}, Appendix~\ref{app:invariant_subspace}).
\end{remark}

\begin{figure*}[t]
    \centering
    \includegraphics[width=\textwidth]{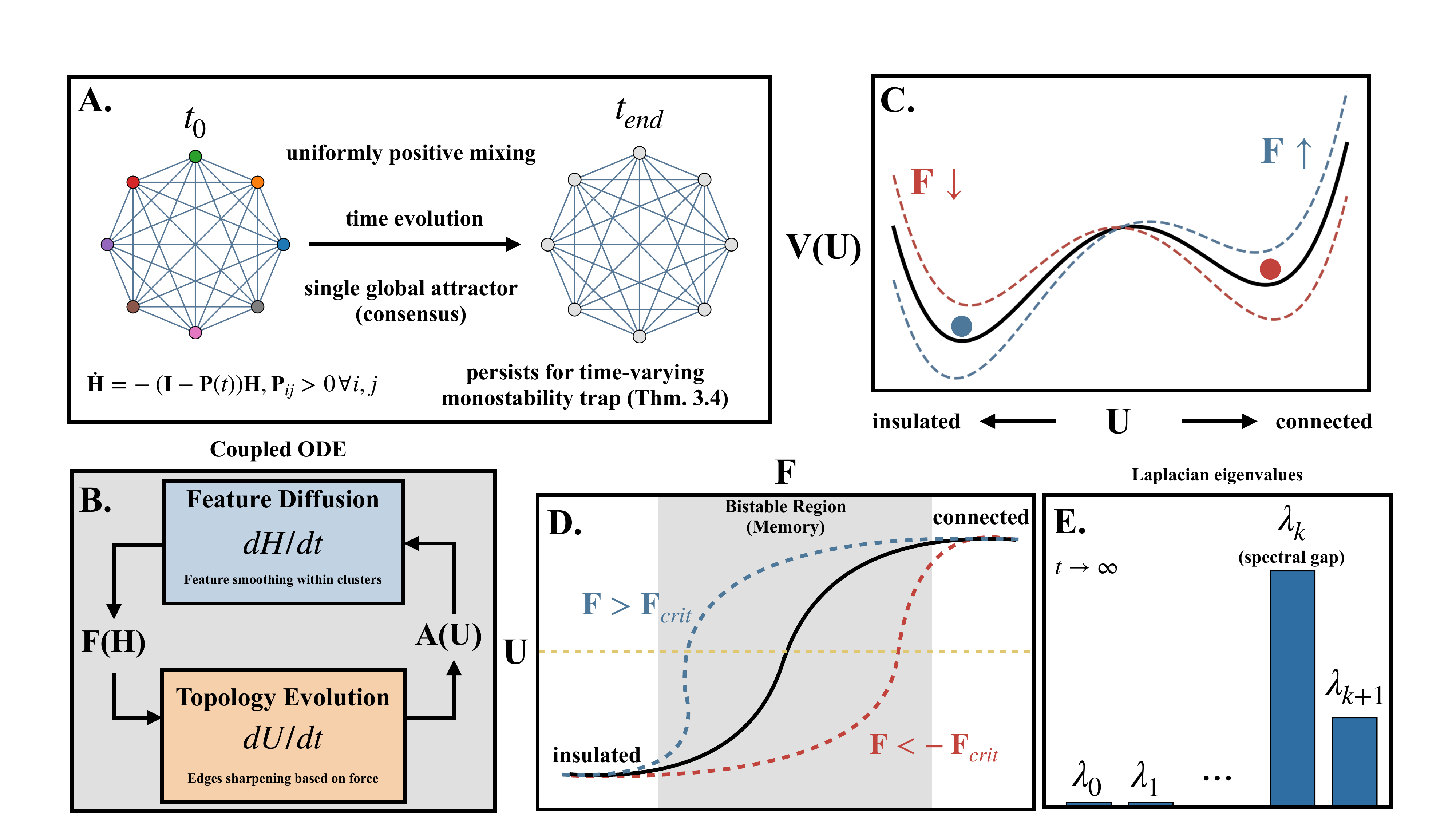}
    \caption{\textbf{HGODE mechanism and theory scope.} (A) Positive irreducible or uniformly positive diffusion-style Graph ODEs contract to a rank-one
    consensus subspace in the long-time regime. (B) HGODE makes topology a dynamical state by
    coupling feature diffusion with edge-potential evolution. (C--D) A feature-conditioned force tilts a
    double-well landscape, polarizing candidate edges into connected/insulated phases with hysteretic
    memory. (E) Edge polarization can move propagation toward block-structured mixing; in the idealized limit this yields cluster-wise invariant subspaces (\eg, multiple near-zero Laplacian eigenvalues) instead of global consensus.}
    \label{fig:framework}
\end{figure*}

\subsection{Beyond time-invariant mixing}
\label{sec:time_varying_main}
Theorem~\ref{thm:monostability} is stated for a time-invariant operator, but
the same collapse mechanism also applies to dense time-varying attention. In
Appendix~\ref{app:tv_proofs}, we prove that if a row-stochastic $\mathbf{P}(t)$
is piecewise continuous and uniformly positive, $\mathbf{P}_{ij}(t)\ge
    \alpha>0$ for all pairs and times, then each feature coordinate contracts in
diameter as
\begin{equation}
    \label{eq:diam_contraction}
    \mathrm{diam}(\mathbf{H}_{\cdot k}(t))
    \le e^{-2\alpha t}\,\mathrm{diam}(\mathbf{H}_{\cdot k}(0)),\qquad k=1,\dots,m,
\end{equation}
and the dynamics converge to a fixed rank-one consensus state. The argument is
trajectory-wise, so it also covers state-dependent $\mathbf{P}(\mathbf{H}(t),t)$
satisfying these conditions, but not arbitrary nonlinear feature maps, source
terms, or normalization layers. Dense softmax attention with temperature
$\tau_{\mathrm{attn}}$ satisfies it with
$\alpha\ge N^{-1}\exp(-2B^2/\tau_{\mathrm{attn}})$ for
$B\defeq\max_j\|\mathbf{h}_j(0)\|_2$, since features stay in their initial
convex hull (Lemma~\ref{lem:convex_hull}); larger attention temperature thus
strengthens rather than removes the consensus pressure. This result targets dense/full
support soft mixing; explicit sparsification or masking can avoid uniform
positivity, but then the asymptotic behavior depends on whether the induced
support becomes reducible.

\subsection{Hysteresis via Double-Well Potentials}
\label{subsec:hysteresis}

A natural way to obtain persistent support-level intervention is to make
topology dynamics \emph{bistable}. We model each directed edge $e_{ij}$ via a
latent order parameter $u_{ij}(t)$ evolving on a non-convex Landau energy
landscape with two stable structural phases: \emph{Connected} and
\emph{Insulated}, rather than as a directly optimized propagation weight. This
is a binary decision at the edge level: an interaction should either support
mixing or suppress it, and multi-cluster structure arises from the joint pattern
of many binary edge states. The quartic double-well is the lowest-order smooth
bistable normal form with an analytic saddle-node threshold; other smooth
bistable potentials could be substituted
(Appendices~\ref{app:bistability_candidate_pool} and~\ref{app:double_well_normal_form}).

\begin{definition}[Landau Double-Well Potential for Connectivity]
    \label{def:landau}
    Let $u_{ij}\in\mathbb{R}$ be the latent order parameter associated with the directed edge $e_{ij}$.
    Given the node features $\mathbf{h}_i, \mathbf{h}_j$, we define the driving force $\mathcal{F}_{ij}(t) \defeq \mathcal{F}(\mathbf{h}_i(t), \mathbf{h}_j(t))$.
    We define the local Landau potential (energy) for each edge as
    \begin{equation}
        \label{eq:landau_potential}
        V(u_{ij}; \mathcal{F}_{ij})
        \;\defeq\;
        \underbrace{\frac{1}{4}u_{ij}^4 - \frac{1}{2}u_{ij}^2}_{\text{intrinsic bistability}}
        \;-\;
        \underbrace{\mathcal{F}_{ij}\,u_{ij}\vphantom{\frac{1}{4}u}}_{\text{coupling field}},
    \end{equation}
    where the intrinsic term yields two degenerate minima at $u_{ij}=\pm 1$ in the absence of external drive ($\mathcal{F}_{ij}=0$),
    and the topological force $\mathcal{F}_{ij}$ acts as an external bias that breaks the symmetry.
\end{definition}

\textbf{Gradient-flow dynamics.}
We evolve the edge state by gradient descent on the potential:
\begin{equation}
    \label{eq:bistable_dynamics}
    \frac{du_{ij}}{dt}
    \;=\;
    -\frac{\partial V(u_{ij};\mathcal{F}_{ij})}{\partial u_{ij}}
    \;=\;
    u_{ij}-u_{ij}^3+\mathcal{F}_{ij}.
\end{equation}
The cubic nonlinearity in~\eqref{eq:bistable_dynamics} induces \textbf{hysteresis}:
the state of an edge depends not only on the instantaneous force $\mathcal{F}_{ij}$ but also on
its history (\ie, which well it currently occupies), thereby providing a form of structural memory.

\begin{proposition}[Bistability, saddle-node transitions, and hysteresis]
    \label{prop:hysteresis}
    For each edge $(i,j)$, the equilibria of~\eqref{eq:bistable_dynamics} are given by the roots of:
    \begin{equation}
        \label{eq:cubic_fixed_points}
        u_{ij}^3 - u_{ij} = \mathcal{F}_{ij}
    \end{equation}
    Define the critical force magnitude:
    \begin{equation}
        \label{eq:fcrit}
        \mathcal{F}_{\mathrm{crit}} \;\defeq\; \frac{2}{3\sqrt{3}}.
    \end{equation}
    Then the system undergoes \emph{saddle-node (fold) bifurcations} at
    $\mathcal{F}_{ij}=\pm \mathcal{F}_{\mathrm{crit}}$, with the following regimes:
    \begin{enumerate}
        \item \textbf{Bistable regime ($|\mathcal{F}_{ij}|<\mathcal{F}_{\mathrm{crit}}$).}
              Eq.~\ref{eq:cubic_fixed_points} admits two stable equilibria (attractors) and one unstable equilibrium (repeller). Consequently, an edge initialized in the \emph{Connected} well
              ($u_{ij}\approx +1$) resists switching to \emph{Insulated} even if $\mathcal{F}_{ij}$ decreases moderately,
              and vice versa, yielding robustness to noise and transient perturbations.

        \item \textbf{Monostable regime ($|\mathcal{F}_{ij}|>\mathcal{F}_{\mathrm{crit}}$).}
              Only one real fixed point remains. The barrier disappears via a saddle-node transition, forcing a rapid,
              deterministic phase change to the single remaining attractor.
    \end{enumerate}
\end{proposition}

\textbf{From latent states to effective propagation.}
A smooth monotone gate (Eq.~\eqref{eq:A_from_U}) maps the Connected well to high and the Insulated well to near-zero propagation weight on $\mathcal{E}_{cand}$, so the forward model stays differentiable; near $|\mathcal{F}_{ij}|\approx \mathcal{F}_{\mathrm{crit}}$ the potential can move rapidly but smoothly, which adaptive solvers (\texttt{dopri5}) integrate stably. For the scaled vector field $(1-\lambda)u_{ij}-u_{ij}^3+\mathcal{F}_{ij}$ used by HGODE, the threshold becomes $\mathcal{F}_{\mathrm{crit}}(\lambda)$ in Eq.~\eqref{eq:fcrit_lambda}.

\section{The Hysteresis Graph ODE Framework}
\label{sec:framework}

We now instantiate the coupled feature--topology ODE designed to alter the
long-time mixing structure analyzed in Sec.~\ref{sec:monostable_trap}. The
overview of the proposed HGODE is shown in Figure~\ref{fig:framework}.

\subsection{Coupled Evolutionary Dynamics}
\label{sec:dynamics}

HGODE governs the co-evolution of node features $\mathbf{H}(t)$ and latent edge
potentials $\mathbf{U}(t)$ on a candidate pool $\mathcal{E}_{cand}$. The
coupled dynamics are
\begin{equation}
    \begin{cases}
        \begin{aligned}
            \tau_{feat}\,\dfrac{d\mathbf{H}(t)}{dt}
             & =
            \mathcal{G}_{\phi}\!\big(\mathbf{H}(t), \mathbf{A}(t)\big)
            -\gamma\,\mathbf{H}(t),
            \\[6pt]
            \tau_{topo}\,\dfrac{d\mathbf{U}(t)}{dt}
             & =
            (1-\lambda)\,\mathbf{U}(t)
            -\mathbf{U}(t)^{\odot 3}
            +\mathcal{F}_{\theta}\!\big(\mathbf{H}(t)\big),
        \end{aligned}
    \end{cases}
    \label{eq:hgode_system}
\end{equation}
where $\mathcal{G}_{\phi}$ is a graph neural operator, $\gamma\ge 0$ is a self-decay coefficient used for finite-time stabilization, and $\lambda\in[0,1)$ controls the double-well depth. Eq.~\eqref{eq:hgode_system} is a framework template: $\mathcal{G}_{\phi}$ can be any graph operator, whereas our convergence results (Theorems~\ref{thm:monostability} and~\ref{thm:time_varying_consensus}, Proposition~\ref{prop:blockwise_consensus}) cover only the consensus form, in which the feature right-hand side reduces to $(\mathbf{P}(\mathbf{A})-\mathbf{I})\mathbf{H}$, under their stated assumptions. Any additional net decay of $\mathbf{H}$ without a source term drives even block-wise consensus values toward zero, so the decay is treated as an experimental stabilization device rather than part of the collapse theory. Sec.~\ref{sec:experiments} evaluates one nonlinear instantiation of $\mathcal{G}_{\phi}$, for which we make no convergence claim.

\textbf{Continuous depth.}
The time $t$ plays the role of network depth, and $\mathbf{H}(T)$ is passed to the task head. A forward-Euler step of size $\Delta t$ gives $\mathbf{H}^{(\ell+1)}=\mathbf{H}^{(\ell)}+\frac{\Delta t}{\tau_{feat}}[\mathcal{G}_{\phi}(\mathbf{H}^{(\ell)},\mathbf{A}(\mathbf{U}^{(\ell)}))-\gamma\mathbf{H}^{(\ell)}]$, and analogously for $\mathbf{U}^{(\ell)}$, so HGODE is a residual message-passing network whose parameters are shared across depth and whose edge weights on the fixed candidate support change with depth. We integrate Eq.~\eqref{eq:hgode_system} with adaptive \texttt{dopri5}, whose function evaluations are numerical evaluations of one vector field rather than separately parameterized layers.

For the scaled topology equation in Eq.~\ref{eq:hgode_system}, the intrinsic
wells are at $\pm\sqrt{1-\lambda}$ and the saddle-node force threshold is
\begin{equation}
    \label{eq:fcrit_lambda}
    \mathcal{F}_{\mathrm{crit}}(\lambda)
    \defeq
    \frac{2(1-\lambda)^{3/2}}{3\sqrt{3}}.
\end{equation}
With $a\defeq1-\lambda$, the rescaling $\bar u=u/\sqrt a$, $\bar t=a t/\tau_{topo}$,
$\bar{\mathcal{F}}=\mathcal{F}/a^{3/2}$ maps each scalar topology equation to the canonical form
$d\bar u/d\bar t=\bar u-\bar u^3+\bar{\mathcal{F}}$ of Eq.~\eqref{eq:bistable_dynamics}, so
Proposition~\ref{prop:hysteresis} applies with threshold $\mathcal{F}_{\mathrm{crit}}(\lambda)$. The barrier
between the wells has height $a^2/4$: increasing $\lambda$ lowers both the barrier and the switching
threshold, making phase changes easier but hysteretic memory weaker; $\lambda=0$ recovers the canonical
case (Appendix~\ref{app:double_well_normal_form}).

\textbf{Force field.}
$\mathcal{F}_{\theta}(\mathbf{H}(t))\in\mathbb{R}^{N\times N}$ is a matrix-valued force field whose entries depend on feature compatibility:
\[
    \big(\mathcal{F}_{\theta}(\mathbf{H})\big)_{ij}=\mathcal{F}_{\theta}(\mathbf{h}_i,\mathbf{h}_j).
\]
Its parameterization and margin-inducing objective are specified in
Sec.~\ref{sec:force_margin}.

\textbf{From latent potentials to effective propagation.}
HGODE uses a sigmoid-only forward gate on the candidate set:
\begin{equation}
    \label{eq:A_from_U}
    \mathbf{A}_{ij}(t)
    =
    \sigma\!\left(\frac{\mathbf{U}_{ij}(t)}{\tau}\right)
    \mathbf{1}\big[(i,j)\in\mathcal{E}_{cand}\big].
\end{equation}
Connected wells map to high and insulated wells to near-zero but nonzero weights. Eq.~\eqref{eq:A_from_U} is the only adjacency in the forward pass and backpropagation, which keeps the coupled ODE differentiable; sigmoid-only HGODE therefore yields finite-horizon or metastable suppression of harmful mixing, whereas exact asymptotic avoidance requires a reducible effective support.

For theory diagnostics and post-hoc evaluation only, we may inspect the
thresholded support projection
\begin{equation}
    \label{eq:evaluation_support_projection}
    \mathcal{E}_{\eta}(t)
    =
    \{(i,j)\in\mathcal{E}_{cand}:\mathbf{A}_{ij}(t)>\eta\},
    \qquad \eta>0.
\end{equation}
It is never used in the forward ODE or in training, and represents an idealized saturated regime in which near-zero sigmoid weights are treated as absent.

\begin{proposition}[Block-wise consensus under an evaluation support projection]
    \label{prop:blockwise_consensus}
    Assume \(\gamma=0\). Consider the idealized projected diffusion flow
    \[
        \dot{\mathbf{H}}(t) = \big(\mathbf{P}_\eta(t) - \mathbf{I}\big)\mathbf{H}(t)
    \]
    obtained by applying the support projection in
    Eq.~\eqref{eq:evaluation_support_projection} after time \(T_0\). Suppose
    \(\mathbf{P}_\eta(t)\) is block diagonal up to permutation, with blocks
    \(\mathbf{P}_1(t),\ldots,\mathbf{P}_K(t)\), and each \(\mathbf{P}_k(t)\) is
    row-stochastic, piecewise continuous, and uniformly positive within its block.
    Then each block converges to its own consensus subspace:
    \[
        \lim_{t\to\infty}
        \left\|\mathbf{H}_{\mathcal{C}_k}(t)-\mathbf{1}\mathbf{y}_{k,\infty}^\top\right\|_F=0,
        \qquad k=1,\dots,K.
    \]
    where $\mathbf{y}_{k,\infty}\in\mathbb R^m$ is the block-specific consensus
    vector. If, in addition, each $\mathbf{P}_k(t)=\mathbf{P}_k$ is time-invariant
    after $T_0$ and irreducible, then
    $\mathbf{y}_{k,\infty}=\boldsymbol{\pi}_k^\top\mathbf{H}_{\mathcal{C}_k}(T_0)$, where $\boldsymbol{\pi}_k$ is the
    stationary distribution of the $k$-th block.
    Unless these block centroids coincide, the idealized projected limit is not
    global rank-one consensus.
\end{proposition}

Appendix~\ref{app:energy_interpretation} gives the gradient-flow (energy) interpretation of both equations.

\subsection{Topological Force Field and Scaling}
\label{sec:force_margin}

The topological force $\mathcal{F}_{\theta}$ couples features to topology. It
determines whether an edge remains in its current well or crosses the
hysteresis threshold $\mathcal{F}_{\mathrm{crit}}(\lambda)$ to switch phase. We
parameterize it as a bounded edge-wise score:
\begin{equation}
    \label{eq:force_def}
    \mathcal{F}_{ij}
    \defeq
    \big(\mathcal{F}_{\theta}(\mathbf{H})\big)_{ij}
    =
    s\cdot\tanh\!\Big(\mathrm{MLP}_{\theta}\big([\mathbf{h}_i\Vert\mathbf{h}_j]\big)\Big),
\end{equation}
where $(i,j)\in\mathcal{E}_{cand}$. When the margin term is active ($\beta>0$), we choose $s\ge\mathcal{F}_{\mathrm{crit}}(\lambda)+\delta$ so the force-margin target is feasible under the bounded $\tanh$ force. Positive forces bias an edge toward the connected well, negative forces bias it toward the insulated well, and magnitudes above $\mathcal{F}_{\mathrm{crit}}(\lambda)$ remove one of the wells through a saddle-node transition. The theory needs no MLP, only a bounded, piecewise-continuous force with, for guaranteed switching, a sustained signed margin (Lemmas~\ref{lem:margin_polarization}--\ref{lem:timevarying_margin_polarization}); Eq.~\eqref{eq:force_def} is one bounded ordered-pair choice allowing $\mathcal{F}_{ij}\neq\mathcal{F}_{ji}$, and $\mathcal{F}_{\mathrm{crit}}(\lambda)$ is fixed by the potential, not by the MLP. Since $\mathcal{F}_{ij}(t)$ follows $\mathbf{H}(t)$, an edge can switch phase in either direction within one forward pass once its force stays beyond the corresponding fold long enough.

\subsection{Margin-Inducing Training Objective}
\label{sec:margin_objective}

Our asymptotic analysis relies on a \emph{force-separation} condition:
compatible pairs should receive positive forces exceeding
$\mathcal{F}_{\mathrm{crit}}(\lambda)$ up to a margin, while incompatible pairs
should receive sufficiently negative forces. Under a sustained force-margin
condition, edge potentials enter and remain in the corresponding stable phase
even when the learned force is time-varying
(Lemmas~\ref{lem:margin_polarization}--\ref{lem:timevarying_margin_polarization},
Appendix~\ref{app:polarization}). We encourage this condition through
\begin{equation}
    \label{eq:force_contra}
    \mathcal{L}_{margin}
    =
    \frac{1}{|\mathcal{P}_s|}\sum_{(i,j)\in\mathcal{P}_s}
    \mathrm{softplus}\big(\mathcal{F}_{\mathrm{crit}}(\lambda)+\delta-\mathcal{F}_{ij}\big)
    +
    \frac{1}{|\mathcal{N}_s|}\sum_{(i,j)\in\mathcal{N}_s}
    \mathrm{softplus}\big(\mathcal{F}_{\mathrm{crit}}(\lambda)+\delta+\mathcal{F}_{ij}\big),
\end{equation}
where $\mathcal{P},\mathcal{N}\subseteq\mathcal{E}_{cand}$ are compatible and incompatible candidate pairs (so the term never adds edges), $\mathcal{P}_s,\mathcal{N}_s$ are random subsets of capped size (empty terms omitted), $\delta>0$ is the margin, and the forces are evaluated by $\mathcal{F}_\theta$ on the pre-ODE representation $\mathbf{H}(0)$. The total training objective is
\begin{equation}
    \label{eq:total_loss}
    \mathcal{L}=\mathcal{L}_{task}+\beta\,\mathcal{L}_{margin}, \text{where}\quad \beta \ge 0
\end{equation}

Pair construction never uses validation or test labels: single-label node tasks
pair training-labeled nodes, multi-label tasks threshold the cosine similarity of
training label vectors (Appendix~\ref{app:multi_label_margin},
Table~\ref{tab:pair_construction}), and other nodes are omitted or assigned by a
label-free pseudo-partition $\{c_i\}$ obtained by $k$-means on the
$\ell_2$-normalized, detached pre-ODE features (Algorithm~\ref{alg:pseudo_partition}). For graph-level tasks it is
computed within each training graph, and graph labels never label node pairs.
Margin pairs and pseudo-partitions are not needed at validation or test time, and
$\mathcal{L}_{task}$ remains the primary supervision signal.

\textbf{Theory scope.}
The edge-level results are mechanism-level (bistability, saddle-node thresholds, and polarization under a sustained margin); Table~\ref{tab:theory_regimes} summarizes the collapse results. We do not prove global convergence of the fully learned coupled system; Appendix~\ref{app:soft_leakage_wellposed} gives finite-horizon leakage control and boundedness of the latent potentials.

\begin{table}[t]
    \caption{Scope of the collapse results, all for consensus-form propagation.}
    \label{tab:theory_regimes}
    \centering
    {\small
        \setlength{\tabcolsep}{3pt}
        \begin{tabular}{@{}>{\raggedright\arraybackslash}p{0.20\linewidth}>{\raggedright\arraybackslash}p{0.25\linewidth}>{\raggedright\arraybackslash}p{0.30\linewidth}>{\raggedright\arraybackslash}p{0.20\linewidth}@{}}
            \toprule
            Regime                                                        & Assumption                                                                & Proven conclusion                                                                        & Scope of claim                 \\
            \midrule
            Irreducible diffusion (Thms.~\ref{thm:monostability}, \ref{thm:time_varying_consensus}) & Fixed irreducible $\mathbf{P}$, or uniformly positive $\mathbf{P}(t)$                 & One global consensus state                                                               & Diagnosis of global collapse   \\
            Support projection (Prop.~\ref{prop:blockwise_consensus})    & Block-diagonal $\mathbf{P}_\eta(t)$, uniformly positive blocks              & Block-wise consensus; distinct block centroids persist                                   & $K$ block modes replace one    \\
            Sigmoid gates (Prop.~\ref{prop:finite_horizon_leakage})       & Cross-block perturbation $\|\mathbf{R}(t)\|_\infty\le\epsilon$ & $\|\mathbf{H}_\epsilon-\mathbf{H}_0\|_\infty\le\epsilon t\|\mathbf{H}(0)\|_\infty$ for $t\le T$ & Finite-horizon mitigation only \\
            \bottomrule
        \end{tabular}}
\end{table}

\subsection{Candidate Pool and Smooth Gating}
\label{sec:active_set}
Large graphs use a sparse candidate pool $\mathcal{E}_{cand}$ to avoid the $\mathcal{O}(N^2)$ cost of all-pairs topology states. The pool combines observed edges, 2-hop common-neighbor proposals, Laplacian-eigenvector (LapPE) cosine proposals, and random exploration pairs, but this is a scalability choice rather than a modeling necessity: small graphs can use $\mathcal{E}_{cand}=\mathcal{V}\times\mathcal{V}$, equivalent to neighbor ratio $1$. The pool is built from the input graph alone (no labels or split masks), observed edges start with positive and new proposals with negative potentials, and the pool stays fixed during each ODE solve; for graph-level tasks it is built within each graph, including unseen test graphs. Algorithm~\ref{alg:candidate_pool} gives the exact construction, and Appendix~\ref{app:candidate_pool_coverage} shows why small random ratios still cover long-distance relation types.

\section{Experiments}
\label{sec:experiments}

\textbf{Instantiated operator.} For an ordered candidate $i\leftarrow j$ (receiver $i$, sender $j$), the implemented gate and transition weights are
\begin{equation}
    \label{eq:implemented_gate}
    g_{ij}=\mu\,\sigma\big(u_{ij}/\tau_g\big),\quad
    \tilde g_{ij}=g_{ij}+\epsilon_{\mathrm{loop}}\mathbf{1}[i=j],\quad
    p_{ij}=\frac{\tilde g_{ij}}{\sum_{k}\tilde g_{ik}+\epsilon_{\mathrm{deg}}}\quad\text{(all at time $t$)},
\end{equation}
\ie, Eq.~\eqref{eq:A_from_U} with a scalar multiplier $\mu$ and temperature $\tau_g$ (annealable over training steps, fixed within a solve), self-loops $\epsilon_{\mathrm{loop}}=10^{-3}$, and a guard $\epsilon_{\mathrm{deg}}=10^{-6}$. With $\mathbf{X}^{(0)}=\mathbf{H}(t)$ and $\ell\in\{0,1\}$,
\begin{equation}
    \label{eq:implemented_operator}
    \mathbf{Z}^{(\ell)}=\tanh\!\Big(\mathrm{LN}_{\ell}\big(\mathbf{P}(t)(\mathbf{X}^{(\ell)}\mathbf{W}_{\ell}+\mathbf{1}\mathbf{b}_{\ell}^{\top})\big)\Big),\qquad
    \mathbf{X}^{(\ell+1)}=(1-\eta_{\ell})\mathbf{X}^{(\ell)}+\eta_{\ell}\mathbf{Z}^{(\ell)},
\end{equation}
with a learned scalar $\eta_{\ell}=\sigma(r_{\ell})$, and $\mathcal{G}_{\phi}(\mathbf{H}(t),\mathbf{A}(t))\defeq\mathbf{X}^{(2)}$ in Eq.~\eqref{eq:hgode_system}; no dropout is sampled inside the ODE. Learned value maps, LayerNorm, $\tanh$, and residual interpolation make this field non-consensus in general, so the convergence results of Sec.~\ref{sec:monostable_trap} are not claimed for it. An encoder (Linear--LN--GELU--Dropout--Linear--LN--$\tanh$, optional positional encodings and virtual node) gives $\mathbf{H}(0)$; $\mathbf{H}(T)$ passes through Linear--LN--GELU--Dropout to $\mathbf{H}^{post}$ and then a node head fusing $[\mathbf{H}(0)\Vert\mathbf{H}^{post}]$ with an $\mathbf{H}^{post}$ residual, or a mean/max-pooling graph head (Appendix~\ref{app:architecture_algorithm}). Each full graph (node tasks) or packed mini-batch (graph tasks) is integrated over a validation-selected horizon $[0,T]$.

\subsection{Synthetic mechanism checks}

\begin{figure}[tb]
    \centering
    \includegraphics[width=\linewidth]{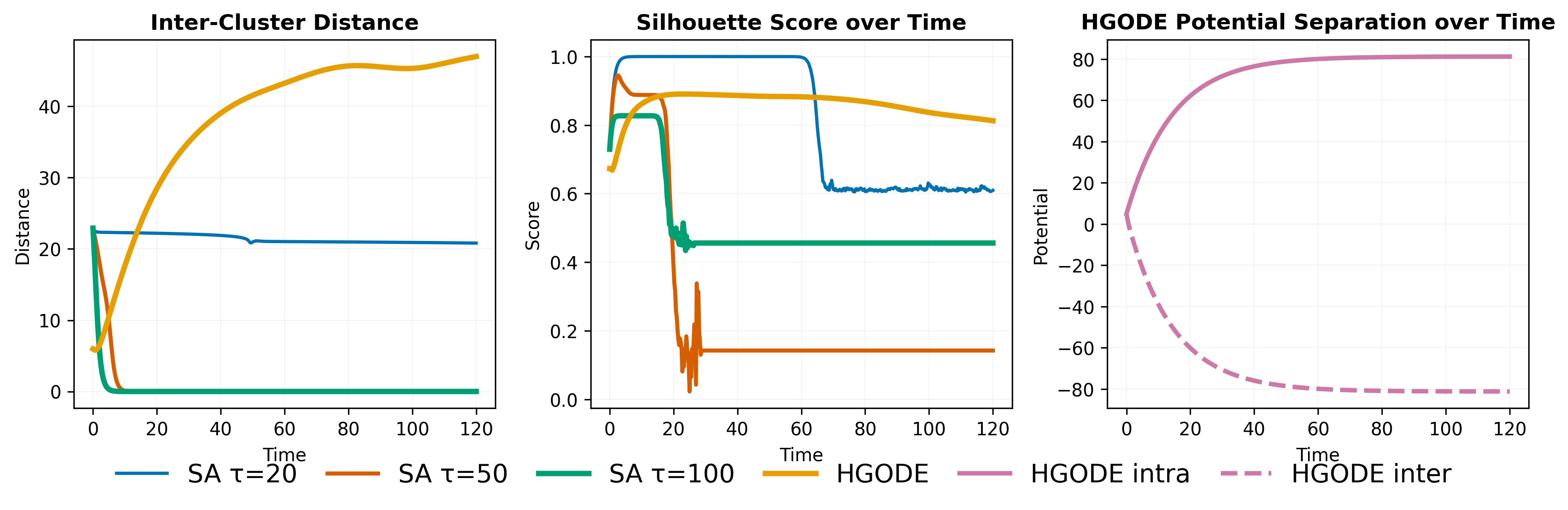}
    \caption{\textbf{Left}: Inter-cluster distance under soft attention collapses as $\tau_{\mathrm{attn}}$ increases, indicating information leakage under increasingly diffuse mixing. \textbf{Middle}: The silhouette score decreases under soft attention at large $\tau_{\mathrm{attn}}$, whereas HGODE remains stable at long time horizons. \textbf{Right}: An implementation-scale HGODE potential statistic polarizes into positive
        (intra) and negative (inter) regimes; the diagnostic is interpreted by sign separation.}
    \label{fig:monostability_sa}
\end{figure}

On $K$-block SBM graphs with Gaussian block features, dense soft-attention
Graph ODEs use strictly positive row-stochastic mixing, while HGODE follows
Eq.~\ref{eq:hgode_system} with the sigmoid gate in Eq.~\ref{eq:A_from_U}.
Figure~\ref{fig:monostability_sa} shows that high attention temperature drives
consensus, whereas HGODE preserves cluster separation and polarizes
intra/inter-cluster potentials; full protocols are in
Appendix~\ref{app:synthetic_protocol}.

\subsection{Real-world benchmark evaluation}
\label{sec:realworld}

We evaluate HGODE on Cora~\cite{mccallum2000automating},
Chameleon-filtered~\cite{rozemberczki2021multi,platonov2022critical},
Roman-empire~\cite{platonov2022critical}, ogbn-proteins~\cite{hu2020open},
ZINC~\cite{JMLR:v24:22-0567}, Peptides-func~\cite{dwivedi2022long}, and
ogbg-molpcba~\cite{hu2020open}. Following Platonov et al., we use the
duplicate-removed Chameleon-filtered setting rather than the older unfiltered
split with train--test leakage from duplicated nodes.
Experiments use NVIDIA
A100 GPUs, AdamW, and dataset-specific learning rates; architecture, split
details, and selected-configuration fields are cross-referenced in
Appendix~\ref{app:architecture_algorithm} and Table~\ref{tab:selected_hparams}.

\textbf{Baselines.}
Table~\ref{tb:performance} compares message-passing GNNs (GCN~\cite{kipf2017semisupervised}, GCNII~\cite{chen2020simple}, GCN$^{+}$~\cite{luo2024classic}) with continuous-depth or anti-over-smoothing models (GDE~\cite{poli2019graph}, CGNN~\cite{pmlr-v119-xhonneux20a}, GRAND~\cite{chamberlain2021grand}, Sheaf Diffusion~\cite{bodnar2022neural}, FLODE~\cite{maskey2024fractional}, GREAD~\cite{choi2023gread}, BuNN~\cite{bamberger2025bundle}, ACMP~\cite{wang2022acmp}, FROND~\cite{kang2024frond}, DRAGON~\cite{zhao2024dragon}). Transformer baselines use different global-attention inductive biases and are reported separately in Appendix~\ref{app:transformer_baselines}.

\begin{table}[tb]
    \caption{Real-world benchmark results (mean $\pm$ standard deviation). All entries, including baselines, are our reruns; run counts, seeds, and splits are given in Appendix~\ref{app:reproducibility}. Transformer baselines are in Table~\ref{tab:transformer_baselines}; ``--'' denotes unavailable/not applicable. Top three values are green, blue, and gray. Ablation rows are defined in Sec.~\ref{sec:realworld}.}
    \label{tb:performance}
    \centering
    {\scriptsize
        \setlength{\tabcolsep}{2pt}
        \renewcommand{\arraystretch}{1.03}
        \resizebox{\textwidth}{!}{%
            \begin{tabular}{lccccccc}
                \toprule
                                                & \multicolumn{4}{c}{Node Classification}
                                                & \multicolumn{3}{c}{Graph-Level Tasks}                                                                                                                                                                  \\
                \cmidrule(lr){2-5}\cmidrule(lr){6-8}
                                                & Cora                                    & Chameleon-filt.         & Roman-empire            & ogbn-proteins           & ZINC                     & Peptides-func            & ogbg-molpcba             \\
                Metric                          & Acc.$\uparrow$                          & Acc.$\uparrow$          & Acc.$\uparrow$          & ROC-AUC$\uparrow$       & MAE$\downarrow$          & A.P.$\uparrow$           & A.P.$\uparrow$           \\
                \midrule
                \multicolumn{8}{l}{\textit{Message-passing GNNs}}                                                                                                                                                                                        \\
                GCN                             & 81.42$\pm$0.36                          & 40.27$\pm$2.48          & 73.76$\pm$0.83          & 71.18$\pm$0.27          & 0.359$\pm$0.015          & 0.683$\pm$0.004          & 0.204$\pm$0.002          \\
                GCN\RomanNumeralCaps{2}         & 84.74$\pm$0.10                          & 44.47$\pm$1.48          & 83.24$\pm$0.88          & 75.64$\pm$0.48          & 0.178$\pm$0.032          & 0.699$\pm$0.003          & 0.265$\pm$0.003          \\
                GCN$^{+}$                       & 85.27$\pm$0.55                          & 45.89$\pm$2.86          & 91.20$\pm$0.45          & 77.84$\pm$0.56          & 0.087$\pm$0.012          & \second{0.716$\pm$0.005} & \second{0.269$\pm$0.002} \\
                \midrule
                \multicolumn{8}{l}{\textit{Continuous-depth and anti-over-smoothing baselines}}                                                                                                                                                          \\
                GDE                             & 82.44$\pm$0.81                          & 45.63$\pm$2.47          & 82.37$\pm$0.61          & 74.63$\pm$0.20          & 0.239$\pm$0.012          & 0.632$\pm$0.014          & 0.247$\pm$0.002          \\
                CGNN                            & 83.68$\pm$0.46                          & 51.35$\pm$1.58          & 83.42$\pm$0.79          & 77.12$\pm$0.61          & 0.153$\pm$0.028          & 0.651$\pm$0.037          & 0.252$\pm$0.003          \\
                GRAND                           & 83.61$\pm$0.37                          & 57.72$\pm$1.86          & 83.92$\pm$0.23          & 76.45$\pm$0.42          & 0.147$\pm$0.018          & 0.663$\pm$0.016          & 0.260$\pm$0.003          \\
                Sheaf Diff.                     & 85.53$\pm$0.64                          & 68.16$\pm$1.62          & 85.43$\pm$0.62          & 78.96$\pm$0.56          & 0.148$\pm$0.021          & 0.678$\pm$0.047          & 0.263$\pm$0.002          \\
                FLODE                           & \third{86.44$\pm$1.17}                  & 67.26$\pm$1.37          & 76.32$\pm$0.60          & 79.23$\pm$0.75          & 0.124$\pm$0.034          & 0.674$\pm$0.041          & 0.267$\pm$0.004          \\
                GREAD-BS                        & \second{87.33$\pm$0.84}                 & \third{71.42$\pm$1.78}  & 88.24$\pm$0.63          & 79.21$\pm$0.74          & 0.084$\pm$0.032          & 0.704$\pm$0.046          & 0.257$\pm$0.004          \\
                GREAD-Exp                       & \best{87.94$\pm$0.42}                   & 69.23$\pm$1.14          & 90.14$\pm$0.72          & 78.61$\pm$0.68          & 0.102$\pm$0.047          & 0.711$\pm$0.040          & 0.262$\pm$0.001          \\
                BuNN                            & 86.10$\pm$0.48                          & 69.13$\pm$1.21          & 90.18$\pm$0.68          & 78.92$\pm$0.48          & 0.087$\pm$0.056          & 0.712$\pm$0.067          & 0.245$\pm$0.004          \\
                ACMP                            & 84.76$\pm$0.74                          & 56.24$\pm$2.01          & 87.12$\pm$0.79          & 65.72$\pm$0.72          & 0.124$\pm$0.036          & 0.679$\pm$0.051          & 0.226$\pm$0.005          \\
                FROND                           & 84.67$\pm$0.86                          & \second{71.62$\pm$1.61} & \third{92.19$\pm$0.47}  & \third{80.26$\pm$0.49}  & \second{0.079$\pm$0.028} & 0.698$\pm$0.034          & 0.260$\pm$0.005          \\
                DRAGON                          & 84.30$\pm$0.61                          & 70.14$\pm$1.33          & \best{93.62$\pm$0.54}   & \second{80.46$\pm$0.42} & \third{0.081$\pm$0.012}  & \best{0.724$\pm$0.045}   & \third{0.266$\pm$0.002}  \\
                \midrule
                HGODE (ours)                    & 86.26$\pm$0.78                          & \best{72.56$\pm$1.24}   & \second{92.56$\pm$0.43} & \best{81.24$\pm$0.63}   & \best{0.078$\pm$0.025}   & \third{0.714$\pm$0.022}  & \best{0.278$\pm$0.003}   \\
                \addlinespace[2pt]
                \quad \textit{w/o hysteresis}   & 83.24$\pm$0.32                          & 66.24$\pm$1.26          & 90.24$\pm$0.63          & 75.26$\pm$0.15          & 0.145$\pm$0.032          & 0.671$\pm$0.013          & 0.254$\pm$0.005          \\
                \quad \textit{w/o topo. search} & 84.14$\pm$0.46                          & 70.44$\pm$1.41          & 90.37$\pm$0.86          & 77.19$\pm$0.52          & 0.162$\pm$0.017          & 0.653$\pm$0.041          & 0.262$\pm$0.002          \\
                \quad \textit{w/o force margin} & 84.36$\pm$0.19                          & 61.24$\pm$0.73          & 89.79$\pm$0.41          & 80.24$\pm$0.77          & 0.172$\pm$0.080          & 0.689$\pm$0.034          & 0.260$\pm$0.003          \\
                \bottomrule
            \end{tabular}}
    }
\end{table}

\textbf{Ablations.} \emph{w/o hysteresis} replaces the topology equation by the single-well relaxation $\tau_{topo}\dot{\mathbf{U}}=\mathcal{F}_{\theta}(\mathbf{H})-\mathbf{U}$, the gradient flow of $V_{sw}(u;\mathcal{F})=\frac12u^2-\mathcal{F}u$, whose unique globally stable equilibrium $u^\star=\mathcal{F}$ keeps soft adaptive weighting but removes bistability and history dependence, with the candidate pool, force network, and gates retained. \emph{w/o topo.\ search} keeps only the observed input edges in $\mathcal{E}_{cand}$. \emph{w/o force margin} sets $\beta=0$.

\textbf{Results, solver, and ablations.}
We use \texttt{dopri5} with $10^{-5}$ tolerances; Appendix~\ref{app:efficiency} reports NFEs, memory, and wall-clock time. Table~\ref{tb:performance} is consistent with the theory: HGODE is strongest on \textsc{Chameleon}-filtered, \textsc{ogbn-proteins}, \textsc{ZINC}, and \textsc{ogbg-molpcba}, and second on sparse, chain-like \textsc{Roman-empire}, where heterophily, multi-label relations, or long-range dependencies make a single positive diffusion support prone to harmful mixing. It remains competitive on homophilous \textsc{Cora} and top-three on \textsc{Peptides-func}.
The ablations connect the gains to the mechanism: removing hysteresis hurts most on \textsc{Chameleon}-filtered, \textsc{Roman-empire}, and \textsc{ogbn-proteins}; removing topology search hurts long-range graph tasks such as \textsc{ZINC} and \textsc{Peptides-func}; and removing the force margin hurts \textsc{Chameleon}-filtered, \textsc{Roman-empire}, and \textsc{ZINC}, supporting the usefulness of the force-margin objective.

\section{Conclusion, Limitations, and Broader Impact}
\label{sec:limitations}
We identify the monostability trap in positive irreducible diffusion-style Graph ODEs and propose HGODE, whose bistable edge potentials learn persistent, task-dependent edge reweighting that can approach near-block mixing when such separation is useful. Exact reducible support remains an idealized analysis limit: the strictly positive sigmoid gates mitigate global consensus collapse over the finite integration horizon rather than avoiding it asymptotically.

\textbf{Limitations.} (i) The consensus theory covers consensus-form propagation, not an arbitrary nonlinear $\mathcal{G}_{\phi}$ in Eq.~\eqref{eq:hgode_system}; we establish neither global convergence of the learned coupled system nor convex-hull preservation beyond Eq.~\eqref{eq:directed_diffusion}. (ii) We evaluate one weighted message-passing $\mathcal{G}_{\phi}$; attention-based instantiations remain untested, and although theory and implementation allow asymmetric operators, our benchmarks are largely undirected, so the empirical role of edge direction is open. (iii) Each candidate edge carries a latent state, and the pool size must be tuned per dataset. (iv) Structure learning is not our primary comparison axis; the edge-state mechanism may also benefit dynamic structure learning, and directed graphs and attention-based backbones are natural next steps. 

\textbf{Broader Impact.} Because HGODE learns latent edge states, inferred high- or low-weight relations should not be read as causal without domain validation; in sensitive applications, learned topology can amplify biases in features, labels, or candidate-pool construction, so deployment should include auditing and task-specific safeguards.

\newpage

\begin{ack}
Qinhan Hou is supported by the iCANDOC Precision Cancer Medicine pilot, part of Finland's Ministry of Education and Culture's Doctoral Education Pilot. This work was also supported by the European Union's Horizon Europe research and innovation programme under grant agreement No.~101188432 (DTRIP4H). Views and opinions expressed are however those of the authors only and do not necessarily reflect those of the European Union or the granting authority. Neither the European Union nor the granting authority can be held responsible for them. The authors wish to acknowledge CSC -- IT Center for Science, Finland, for computational resources. We thank the anonymous reviewers and the area chair for their constructive feedback. The authors declare no competing interests.
\end{ack}

{\small
    \bibliographystyle{plainnat}
    \bibliography{ref}
}

\newpage

\appendix
\section{Proofs.}
\label{sec:proofs}

\subsection{Energy interpretation details}
\label{app:energy_interpretation}
The topology equation in Eq.~\ref{eq:hgode_system} is a Euclidean gradient flow conditioned on $\mathbf{H}$: it descends
\begin{equation}
    \mathcal{E}_{topo}(\mathbf{U};\mathbf{H})=
    \sum_{(i,j)\in\mathcal{E}_{cand}}
    \left[
        \frac{1}{4}\mathbf{U}_{ij}^{4}
        -\frac{1-\lambda}{2}\mathbf{U}_{ij}^{2}
        -\big(\mathcal{F}_{\theta}(\mathbf{h}_i,\mathbf{h}_j)\big)\mathbf{U}_{ij}
        \right].
\end{equation}
Thus $\tau_{topo}\dot{\mathbf{U}}$ follows the negative gradient up to the timescale factor. For symmetric diffusion with $\mathcal{G}_{\phi}(\mathbf{H},\mathbf{A})=\mathbf{P}\mathbf{H}-\mathbf{H}$, the feature flow admits the regularized Dirichlet energy
\begin{equation}
    \mathcal{E}_{feat}(\mathbf{H};\mathbf{U})
    =
    \frac{1}{2}\mathrm{Tr}\!\Big(\mathbf{H}^{\top}(\mathbf{I}-\mathbf{P})\mathbf{H}\Big)
    +\frac{\gamma}{2}\|\mathbf{H}\|_{F}^{2}.
\end{equation}
For reversible non-symmetric Markov operators the analogous statement uses the stationary-distribution-weighted inner product or a similarity transform. General directed row-stochastic operators need not be Euclidean gradient flows, which is why the main proofs use consensus contraction.

\subsection{Proof of Theorem~\ref{thm:monostability}.}
\begin{proof}
    The evolution of the node features is governed by the linear differential equation:
    \begin{equation}
        \frac{d\mathbf{H}(t)}{dt} = -\mathbf{L}_{rw} \mathbf{H}(t),
    \end{equation}
    where $\mathbf{L}_{rw} = \mathbf{I} - \mathbf{P}$ is the random-walk Laplacian. Since $\mathbf{P}$ is time-invariant in this analysis, the analytical solution is given by the matrix exponential:
    \begin{equation}
        \mathbf{H}(t) = e^{-t(\mathbf{I} - \mathbf{P})} \mathbf{H}(0).
    \end{equation}

    We analyze the asymptotic behavior of the operator $\mathbf{T}(t) =
        e^{-t(\mathbf{I} - \mathbf{P})}$ via the spectral properties of $\mathbf{P}$.
    Since $\mathbf{P}$ is row-stochastic and irreducible (by the strong
    connectivity assumption), the Perron-Frobenius theorem guarantees the
    following:
    \begin{enumerate}
        \item The spectral radius is $\rho(\mathbf{P}) = 1$, and $\lambda_1 = 1$ is a simple
              eigenvalue.
        \item The right eigenvector associated with $\lambda_1$ is $\mathbf{1}$ (since
              $\mathbf{P}\mathbf{1} = \mathbf{1}$).
        \item The left eigenvector associated with $\lambda_1$ is the unique stationary
              distribution $\boldsymbol{\pi}$.
        \item All other eigenvalues $\lambda_i$ (for $i=2, \dots, N$) satisfy
              $\mathrm{Re}(\lambda_i)<1$. Irreducibility makes $\lambda=1$ simple; periodic
              chains may have other eigenvalues with $|\lambda_i|=1$, but their real parts
              are still strictly below $1$, which is sufficient for continuous-time decay.
    \end{enumerate}

    Let $\mathbf{P} = \mathbf{V}\boldsymbol{\Lambda}\mathbf{V}^{-1}$ be the
    eigendecomposition of $\mathbf{P}$ (assuming diagonalizability for exposition;
    the argument holds generally using Jordan canonical forms, see Appendix~\ref{app:jordan}). The eigenvalues of
    the Laplacian $\mathbf{L}_{rw}$ are given by $\mu_i = 1 - \lambda_i$. Using the
    spectral expansion, we can express the evolution of the features as:
    \begin{equation}
        \mathbf{H}(t) = \left( \sum_{i=1}^N e^{-t(1 - \lambda_i)} \mathbf{u}_i \mathbf{v}_i^\top \right) \mathbf{H}(0),
    \end{equation}
    where $\mathbf{u}_i$ and $\mathbf{v}_i^\top$ are the corresponding right and left eigenvectors.
    In the non-diagonalizable case, each Jordan block for $\lambda_i$ contributes only polynomial factors multiplying $e^{-t(1-\lambda_i)}$.

    Separating the dominant mode ($i=1$) from the non-dominant modes ($i \ge 2$):
    \begin{equation}
        \mathbf{H}(t) = e^{-t(1 - 1)} \mathbf{1}\boldsymbol{\pi}^\top \mathbf{H}(0) + \sum_{i=2}^N e^{-t(1 - \lambda_i)} \mathbf{u}_i \mathbf{v}_i^\top \mathbf{H}(0).
    \end{equation}

    For the dominant mode, the exponential term is $e^0 = 1$. For all non-dominant
    modes $i \ge 2$, since $\mathrm{Re}(\lambda_i) < 1$, we have $\mathrm{Re}(1 -
        \lambda_i) > 0$. Consequently, the exponential terms decay asymptotically:
    \begin{equation}
        \lim_{t \to \infty} \left| e^{-t(1 - \lambda_i)} \right| = \lim_{t \to \infty} e^{-t \cdot \mathrm{Re}(1 - \lambda_i)} = 0.
    \end{equation}

    Taking the limit $t \to \infty$ for the full system:
    \begin{equation}
        \lim_{t \to \infty} \mathbf{H}(t) = \mathbf{1}\boldsymbol{\pi}^\top \mathbf{H}(0) + \mathbf{0} = \mathbf{1} \left( \boldsymbol{\pi}^\top \mathbf{H}(0) \right).
    \end{equation}
    Thus, the system converges to a state where every row of $\mathbf{H}$ is identical to the weighted centroid of the initial features, completing the proof.
\end{proof}

\subsection{Jordan-form refinement of Theorem~\ref{thm:monostability}}
\label{app:jordan}
Let $\mathbf{P}=\mathbf{V}\mathbf{J}\mathbf{V}^{-1}$ be a Jordan decomposition and let $\mathcal{J}_{nc}$ denote
the Jordan blocks of the non-consensus eigenvalues $\lambda_i$, $i\ge2$. By the Perron--Frobenius
theorem, $\lambda_1=1$ is algebraically simple, so its block is $1\times1$ with spectral projector
$\mathbf{1}\boldsymbol{\pi}^\top$. For a block $\mathbf{J}_s(\lambda)=\lambda\mathbf{I}_s+\mathbf{N}_s$ of size $s$, where
$\mathbf{N}_s$ is nilpotent, $\mathbf{I}_s-\mathbf{J}_s(\lambda)=(1-\lambda)\mathbf{I}_s-\mathbf{N}_s$ and the two terms
commute, so
\[
    e^{-t(\mathbf{I}_s-\mathbf{J}_s(\lambda))}
    =e^{-t(1-\lambda)}\sum_{r=0}^{s-1}\frac{t^r}{r!}\,\mathbf{N}_s^r .
\]
Define $q\defeq\max_{\mathbf{J}\in\mathcal{J}_{nc}}\big(s(\mathbf{J})-1\big)$, where $s(\mathbf{J})$ is the size of block
$\mathbf{J}$. Since $|e^{-t(1-\lambda_i)}|=e^{-t\,\mathrm{Re}(1-\lambda_i)}\le e^{-\Delta t}$ for $i\ge2$, there is a
constant $C$ depending only on $\mathbf{P}$ such that
\[
    \big\|e^{-t(\mathbf{I}-\mathbf{P})}-\mathbf{1}\boldsymbol{\pi}^\top\big\|\le C\,(1+t)^q e^{-\Delta t},\qquad t\ge0,
\]
which is $O(e^{-(\Delta-\varepsilon)t})$ for every $\varepsilon>0$. If $\mathbf{P}$ is diagonalizable, then $q=0$
and the decay rate is exactly $\Delta$.

\subsection{Proof of Proposition~\ref{prop:blockwise_consensus}}
\begin{proof}
    After $T_0$, the idealized projected support has no edges between distinct blocks, so the projected row-stochastic transition matrix is block diagonal up to a permutation:
    \[
        \mathbf{P}_{\eta}(t)=\mathrm{diag}(\mathbf{P}_1(t),\dots,\mathbf{P}_K(t)).
    \]
    The projected feature dynamics therefore decouple across blocks. For each block
    $\mathcal C_k$, the assumed row-stochasticity, piecewise continuity, and
    uniform positivity of $\mathbf{P}_k(t)$ allow
    Theorem~\ref{thm:time_varying_consensus} to be applied to the restricted
    feature matrix $\mathbf{H}_{\mathcal C_k}$, yielding
    \[
        \lim_{t\to\infty}
        \left\|\mathbf{H}_{\mathcal{C}_k}(t)-\mathbf{1}\mathbf{y}_{k,\infty}^\top\right\|_F=0.
    \]
    This proves block-wise convergence to separate consensus subspaces. If, in
    addition, $\mathbf{P}_k(t)=\mathbf{P}_k$ is fixed after $T_0$ and irreducible,
    Theorem~\ref{thm:monostability} applies within block $k$ and gives
    \[
        \lim_{t\to\infty}\mathbf{H}_{\mathcal{C}_k}(t)
        =\mathbf{1}\big(\boldsymbol{\pi}_k^\top\mathbf{H}_{\mathcal{C}_k}(T_0)\big).
    \]
    Because no cross-block dynamics remain in the projected system, different
    blocks retain their own stationary centroids. The projected limit belongs to the global-consensus subspace if and only if all block consensus vectors coincide.
\end{proof}

\subsection{Invariant consensus subspaces}
\label{app:invariant_subspace}

\begin{proposition}[Invariant consensus subspace]
    \label{prop:invariant_subspace}
    Let $\mathbf{P}\in\mathbb{R}^{N\times N}$ be row-stochastic.
    \begin{enumerate}
        \item If $\mathbf{P}$ is irreducible, then $\ker(\mathbf{I}-\mathbf{P})=\mathrm{span}\{\mathbf{1}\}$.
        \item If, up to a permutation of the nodes, $\mathbf{P}=\mathrm{diag}(\mathbf{P}_1,\dots,\mathbf{P}_K)$ with
              irreducible blocks on $\mathcal{C}_1,\dots,\mathcal{C}_K$, then
              $\ker(\mathbf{I}-\mathbf{P})=\mathrm{span}\{\mathbf{1}_{\mathcal{C}_1},\dots,\mathbf{1}_{\mathcal{C}_K}\}$.
        \item For $c\neq0$, $\rho\neq1$, and $h\neq0$, the matrices $c(\mathbf{I}-\mathbf{P})$,
              $\mathbf{I}-[(1-\rho)\mathbf{P}+\rho\mathbf{I}]$, and $\mathbf{I}-\mathbf{M}_h$ with
              $\mathbf{M}_h\defeq\mathbf{I}-h(\mathbf{I}-\mathbf{P})$ have the same kernel as $\mathbf{I}-\mathbf{P}$. In cases~1
              and~2, so does $\mathbf{I}-\mathbf{P}'$ for every row-stochastic $\mathbf{P}'$ with the same support as $\mathbf{P}$.
        \item In case~2 with $K\ge2$, every point of $\mathcal{C}_{\mathrm{global}}$ is an equilibrium of
              $\dot{\mathbf{H}}=(\mathbf{P}-\mathbf{I})\mathbf{H}$, but $\mathcal{C}_{\mathrm{global}}$ is not attractive. The
              limit $\lim_{t\to\infty}\mathbf{H}_{\mathcal{C}_k}(t)=\mathbf{1}\boldsymbol{\pi}_k^\top\mathbf{H}_{\mathcal{C}_k}(0)$, with
              $\boldsymbol{\pi}_k$ the stationary distribution of $\mathbf{P}_k$, lies in $\mathcal{C}_{\mathrm{global}}$ if and only
              if all block centroids $\boldsymbol{\pi}_k^\top\mathbf{H}_{\mathcal{C}_k}(0)$ coincide, and for every
              $\varepsilon>0$ there are initial states within Frobenius distance $\varepsilon$ of
              $\mathcal{C}_{\mathrm{global}}$ whose limits are at positive distance from it.
    \end{enumerate}
\end{proposition}

\begin{proof}
    (1) Since $\mathbf{I}-\mathbf{P}$ is real, it suffices to consider real $\mathbf{x}$ with
    $(\mathbf{I}-\mathbf{P})\mathbf{x}=\mathbf{0}$. Let $M=\max_i x_i$ and $S=\{i:x_i=M\}$. For $i\in S$,
    $M=x_i=\sum_j\mathbf{P}_{ij}x_j\le M$, with equality only if $x_j=M$ whenever $\mathbf{P}_{ij}>0$; hence $S$ is
    closed under the edges of the support graph, and strong connectivity gives $S=\mathcal{V}$, \ie,
    $\mathbf{x}=M\mathbf{1}$. Conversely, $\mathbf{P}\mathbf{1}=\mathbf{1}$.
    (2) The equation $(\mathbf{I}-\mathbf{P})\mathbf{x}=\mathbf{0}$ decouples into
    $(\mathbf{I}-\mathbf{P}_k)\mathbf{x}_{\mathcal{C}_k}=\mathbf{0}$, $k=1,\dots,K$, and part~1 applies to each irreducible block.
    (3) The three matrices equal $c(\mathbf{I}-\mathbf{P})$, $(1-\rho)(\mathbf{I}-\mathbf{P})$, and $h(\mathbf{I}-\mathbf{P})$.
    Irreducibility, and block-diagonal structure with irreducible blocks, depend only on the support, so a
    row-stochastic $\mathbf{P}'$ with the same support falls in the same case and parts~1--2 give the same kernel.
    (4) Since $\mathbf{P}\mathbf{1}=\mathbf{1}$, $(\mathbf{P}-\mathbf{I})\mathbf{1}\mathbf{y}^\top=\mathbf{0}$ for every $\mathbf{y}$. The flow
    decouples across blocks, and Theorem~\ref{thm:monostability} applied to each block gives the stated limit,
    which has the form $\mathbf{1}\mathbf{y}^\top$ exactly when all block centroids coincide. For the last claim, fix
    $\mathbf{y}\in\mathbb{R}^m$, $\delta>0$, and a unit vector $\mathbf{e}\in\mathbb{R}^m$, and set
    $\mathbf{H}(0)=\mathbf{1}\mathbf{y}^\top+\delta\,\mathbf{1}_{\mathcal{C}_1}\mathbf{e}^\top$, which lies within distance
    $\delta\sqrt{|\mathcal{C}_1|}$ of $\mathcal{C}_{\mathrm{global}}$. Because $\boldsymbol{\pi}_k^\top\mathbf{1}=1$, the limit equals
    $\mathbf{1}(\mathbf{y}+\delta\mathbf{e})^\top$ on $\mathcal{C}_1$ and $\mathbf{1}\mathbf{y}^\top$ on the other blocks. Minimizing
    $|\mathcal{C}_1|\,\|\mathbf{y}+\delta\mathbf{e}-\mathbf{z}\|_2^2+(N-|\mathcal{C}_1|)\,\|\mathbf{y}-\mathbf{z}\|_2^2$ over
    $\mathbf{z}\in\mathbb{R}^m$ shows that the limit is at distance $\delta\sqrt{|\mathcal{C}_1|(N-|\mathcal{C}_1|)/N}>0$ from
    $\mathcal{C}_{\mathrm{global}}$.
\end{proof}

\paragraph{Relation to discrete depth and source terms.}
Because $\mathbf{I}-\mathbf{M}_h=h(\mathbf{I}-\mathbf{P})$, an explicit Euler layer, and hence any stack of such layers,
has the same consensus subspace as the continuous flow; the two differ only in how the non-consensus
modes decay, $e^{-t(1-\lambda_i)}$ versus $(1-h(1-\lambda_i))^{L}$ after $L$ layers. The rank-collapse analysis
of \citet{roth2024rank} studies this aggregation-driven mechanism across discrete depth. Our statements
concern the continuous-time flow, including directed, non-reversible, and uniformly positive time-varying
operators (Theorems~\ref{thm:monostability} and~\ref{thm:time_varying_consensus}), and couple the support to
a bistable edge state. GRAND++~\cite{thorpe2022grandpp} adds a source term,
$\dot{\mathbf{H}}=(\mathbf{P}-\mathbf{I})\mathbf{H}+\mathbf{S}$, so equilibria, when they exist, solve
$(\mathbf{I}-\mathbf{P})\mathbf{H}=\mathbf{S}$ and need not be node-constant even though
$\ker(\mathbf{I}-\mathbf{P})=\mathrm{span}\{\mathbf{1}\}$ is unchanged. Support separation instead enlarges
$\ker(\mathbf{I}-\mathbf{P})$ itself. Neither mechanism subsumes the other; within the homogeneous class
$\dot{\mathbf{H}}=(\mathbf{P}-\mathbf{I})\mathbf{H}$, however, only a change of support, to a reducible operator with at least
two closed classes, can remove the single global consensus mode.

\subsection{Soft-gate leakage and coupled-system boundedness}
\label{app:soft_leakage_wellposed}

\begin{proposition}[Finite-horizon leakage under soft gates]
    \label{prop:finite_horizon_leakage}
    Let $\mathbf{P}_{\epsilon}(t)=\mathbf{P}_0(t)+\mathbf{R}(t)$ on $t\in[0,T]$, where
    $\mathbf{P}_0(t)$ is block diagonal and row-stochastic, $\mathbf{P}_{\epsilon}(t)$
    is row-stochastic, and $\|\mathbf{R}(t)\|_{\infty}\le\epsilon$ for all $t$. Let
    $\mathbf{H}_{\epsilon}$ and $\mathbf{H}_0$ solve
    \[
        \dot{\mathbf{H}}_{\epsilon}=(\mathbf{P}_{\epsilon}(t)-\mathbf{I})\mathbf{H}_{\epsilon},
        \qquad
        \dot{\mathbf{H}}_{0}=(\mathbf{P}_0(t)-\mathbf{I})\mathbf{H}_0
    \]
    with the same initial condition $\mathbf{H}(0)$. Then
    \[
        \|\mathbf{H}_{\epsilon}(t)-\mathbf{H}_0(t)\|_{\infty}
        \le
        \epsilon t\,\|\mathbf{H}(0)\|_{\infty},
        \qquad 0\le t\le T.
    \]
    For the consensus-form specialisation of HGODE, this bound controls the finite-horizon deviation from ideal block-wise propagation when the realised operators satisfy the stated small-leakage condition.
    Here $\|\cdot\|_\infty$ denotes the induced row-sum norm; for
    $\mathbf{X}\in\mathbb R^{N\times m}$, this means
    $\|\mathbf{X}\|_\infty=\max_i\sum_{\ell=1}^m|\mathbf{X}_{i\ell}|$.
\end{proposition}

\begin{proof}
    Let $\mathbf{D}(t)=\mathbf{H}_{\epsilon}(t)-\mathbf{H}_0(t)$. Then
    \[
        \dot{\mathbf{D}}
        =
        (\mathbf{P}_{\epsilon}(t)-\mathbf{I})\mathbf{D}
        +\mathbf{R}(t)\mathbf{H}_0(t),
        \qquad \mathbf{D}(0)=0.
    \]
    Let $\boldsymbol{\Phi}_{\epsilon}(t,s)$ be the state-transition matrix for
    $\dot{\mathbf{Z}}=(\mathbf{P}_{\epsilon}(t)-\mathbf{I})\mathbf{Z}$. Since
    $\mathbf{P}_{\epsilon}(t)-\mathbf{I}$ is a time-varying Markov generator with
    nonnegative off-diagonal entries and zero row sums,
    $\boldsymbol{\Phi}_{\epsilon}(t,s)$ is nonnegative and row-stochastic. Hence
    $\|\boldsymbol{\Phi}_{\epsilon}(t,s)\|_{\infty}\le1$. The same argument for
    $\mathbf{P}_0(t)-\mathbf{I}$ gives
    $\|\mathbf{H}_0(t)\|_{\infty}\le\|\mathbf{H}(0)\|_{\infty}$. Variation of
    constants yields
    \[
        \mathbf{D}(t)=\int_0^t
        \boldsymbol{\Phi}_{\epsilon}(t,s)\mathbf{R}(s)\mathbf{H}_0(s)\,ds,
    \]
    and therefore
    \[
        \|\mathbf{D}(t)\|_{\infty}
        \le
        \int_0^t
        \|\boldsymbol{\Phi}_{\epsilon}(t,s)\|_{\infty}
        \|\mathbf{R}(s)\mathbf{H}_0(s)\|_{\infty}\,ds
        \le
        \int_0^t \epsilon\,\|\mathbf{H}(0)\|_{\infty}\,ds,
    \]
    which proves the bound.
\end{proof}

\begin{corollary}[Time-dependent leakage]
    \label{cor:leakage_integral}
    Under the assumptions of Proposition~\ref{prop:finite_horizon_leakage}, but with a time-dependent bound
    $\|\mathbf{R}(t)\|_\infty\le\epsilon(t)$ for an integrable $\epsilon(\cdot)$,
    \[
        \|\mathbf{H}_{\epsilon}(t)-\mathbf{H}_0(t)\|_{\infty}\le\|\mathbf{H}(0)\|_{\infty}\int_0^t\epsilon(s)\,ds,\qquad 0\le t\le T.
    \]
    In particular, fix a partition $\{\mathcal{C}_1,\dots,\mathcal{C}_K\}$, write $\mathcal{C}(i)$ for the block of node $i$,
    let $\ell_i(t)\defeq\sum_{j\notin\mathcal{C}(i)}\mathbf{P}_{\epsilon,ij}(t)<1$ be the cross-block mass of row $i$ and
    $\ell_\infty(t)\defeq\max_i\ell_i(t)$, and let $\mathbf{P}_0(t)$ be obtained from $\mathbf{P}_\epsilon(t)$ by deleting the
    cross-block entries and renormalizing each row within its block. Then
    $\|\mathbf{P}_\epsilon(t)-\mathbf{P}_0(t)\|_\infty=2\ell_\infty(t)$, and hence
    \[
        \|\mathbf{H}_{\epsilon}(t)-\mathbf{H}_0(t)\|_{\infty}\le2\,\|\mathbf{H}(0)\|_{\infty}\int_0^t\ell_\infty(s)\,ds .
    \]
\end{corollary}

\begin{proof}
    In the proof of Proposition~\ref{prop:finite_horizon_leakage}, the integrand is bounded by
    $\|\mathbf{R}(s)\|_\infty\|\mathbf{H}(0)\|_\infty\le\epsilon(s)\|\mathbf{H}(0)\|_\infty$, which gives the first claim. For the
    second, $\mathbf{P}_0(t)$ is block diagonal and row-stochastic, because each row keeps its within-block mass
    $1-\ell_i(t)>0$ and divides it by $1-\ell_i(t)$. Row $i$ of $\mathbf{P}_\epsilon(t)-\mathbf{P}_0(t)$ consists of the deleted
    cross-block entries, with absolute sum $\ell_i(t)$, and the within-block entries
    $\mathbf{P}_{\epsilon,ij}(t)\big(1-(1-\ell_i(t))^{-1}\big)$, with absolute sum
    $(1-\ell_i(t))\big((1-\ell_i(t))^{-1}-1\big)=\ell_i(t)$. Taking the maximum over rows gives $2\ell_\infty(t)$. The
    numerical self-loops of Sec.~\ref{notations} keep every diagonal entry positive, so $\ell_i(t)<1$ holds
    automatically.
\end{proof}

\begin{proposition}[Well-posedness and bounded edge potentials]
    \label{prop:well_posed_bounded_potential}
    Assume the candidate set $\mathcal E_{cand}$ is finite,
    $\tau_{feat},\tau_{topo}>0$, $\lambda<1$, the bounded force satisfies
    $|\mathcal F_{ij}(t)|\le F_{\max}$ for all $(i,j)\in\mathcal E_{cand}$ and
    $t\in[0,T]$, and $\mathcal F_{\theta}$ and
    $\mathcal G_{\phi}$ are locally Lipschitz in their arguments. Assume also that
    $\mathcal G_{\phi}$ has at most linear growth in $\mathbf{H}$ uniformly over the
    row-normalized sigmoid gates used by HGODE. Then the coupled HGODE has a unique
    solution on every finite horizon $[0,T]$ for any initial
    $(\mathbf{H}(0),\mathbf{U}(0))$. Moreover, if $R>0$ satisfies
    \[
        R^3-(1-\lambda)R-F_{\max}>0,
    \]
    then every edge potential entering $[-R,R]$ remains in $[-R,R]$, and
    trajectories initialized outside this interval are driven back toward it.
\end{proposition}

\begin{proof}
    The sigmoid gate is smooth and bounded, and the candidate set is finite. Hence
    the coupled vector field is locally Lipschitz. For
    $a\defeq 1-\lambda$, the scalar topology dynamics satisfy
    \[
        \tau_{topo}\dot u = au-u^3+F(t), \qquad |F(t)|\le F_{\max}.
    \]
    At $u=R$, $au-u^3+F(t)\le aR-R^3+F_{\max}<0$, while at $u=-R$, $au-u^3+F(t)\ge
        -aR+R^3-F_{\max}>0$. Thus the vector field points inward at the boundary of
    $[-R,R]$. Since $R^3-aR-F_{\max}>0$ implies $R^2>a$, the cubic term dominates
    outside the interval as well, so trajectories outside are pushed back toward
    the interval. The latent topology state is therefore bounded on finite
    horizons. With $\mathbf{U}$ bounded, the gates are bounded, and the assumed
    linear-growth bound on $\mathcal G_{\phi}$ gives a Gronwall bound for
    $\|\mathbf{H}(t)\|$ on $[0,T]$. Local existence and uniqueness can therefore be
    continued to the full finite horizon.
\end{proof}

\begin{remark}[Latent-potential identifiability]
    The raw potential $\mathbf{U}$ is a latent parameterization of edge phase, not an
    identifiable physical edge strength. Because propagation depends on
    $\sigma(\mathbf{U}/\tau)$, only the induced gates, their signs/polarized phases,
    and thresholded diagnostic supports are interpreted.
\end{remark}

\subsection{Time-varying consensus statements and proofs}
\label{app:tv_proofs}

\begin{assumption}[Uniform Positivity]
    \label{ass:uniform_pos}
    Let $\mathbf{P}(t)$ be a time-varying row-stochastic matrix. We assume that there exists a constant
    $\alpha>0$ (necessarily $\alpha \le 1/N$) such that:
    (i) $\mathbf{P}(t)$ is piecewise continuous in $t$;
    (ii) $\mathbf{P}_{ij}(t) \ge \alpha$ for all $i,j\in\{1,\dots,N\}$ and all $t \ge 0$.
\end{assumption}

We first prove a differential contraction bound, from which the window
contraction statement follows immediately.

\begin{lemma}[Instantaneous diameter contraction]
    \label{lem:instant_contraction}
    Let $\mathbf{x}(t)\in\mathbb{R}^N$ follow the time-varying consensus flow
    \[
        \dot{\mathbf{x}}(t) = -(\mathbf{I} - \mathbf{P}(t))\mathbf{x}(t) = -\mathbf{x}(t) + \mathbf{P}(t)\mathbf{x}(t),
    \]
    where $\mathbf{P}(t)$ is row-stochastic for all $t$ and satisfies uniform
    positivity: $\mathbf{P}_{ij}(t)\ge \alpha$ for all $i,j,t$ with some
    $\alpha>0$. Define $\mathrm{diam}(\mathbf{x})\defeq \max_i \mathbf{x}_i -
        \min_i \mathbf{x}_i$. Then for almost every $t$,
    \[
        \frac{d}{dt}\mathrm{diam}(\mathbf{x}(t)) \le -2\alpha\,\mathrm{diam}(\mathbf{x}(t)).
    \]
    Consequently,
    \[
        \mathrm{diam}(\mathbf{x}(t)) \le e^{-2\alpha t}\,\mathrm{diam}(\mathbf{x}(0)),\qquad \forall t\ge 0.
    \]
\end{lemma}

\begin{proof}
    Let $M(t)\defeq \max_i \mathbf{x}_i(t)$ and $m(t)\defeq \min_i \mathbf{x}_i(t)$.
    Pick indices $i^+(t)\in\arg\max_i \mathbf{x}_i(t)$ and $i^-(t)\in\arg\min_i \mathbf{x}_i(t)$
    (the argument holds for almost every $t$, since $M(t)$ and $m(t)$ are \textit{a.e.} differentiable).

    For the maximizer $i^+$, we have
    \[
        \dot{\mathbf{x}}_{i^+}(t) = -\mathbf{x}_{i^+}(t) + \sum_{j=1}^N \mathbf{P}_{i^+ j}(t)\,\mathbf{x}_j(t).
    \]
    Let $j^-$ be any minimizer such that $\mathbf{x}_{j^-}(t)=m(t)$. By uniform
    positivity, $\mathbf{P}_{i^+ j^-}(t)\ge \alpha$. Since all other
    $\mathbf{x}_j(t)\le M(t)$, we obtain
    \[
        \sum_{j=1}^N \mathbf{P}_{i^+ j}(t)\,\mathbf{x}_j(t)
        \le \alpha m(t) + (1-\alpha)M(t),
    \]
    and hence $\dot M(t)\le -\alpha(M(t)-m(t))$. Similarly, for the minimizer
    $i^-$, using a maximizer $j^+$ gives $\dot m(t)\ge \alpha(M(t)-m(t))$.
    Combining the two inequalities yields
    \[
        \frac{d}{dt}\mathrm{diam}(\mathbf{x}(t))\le -2\alpha\,\mathrm{diam}(\mathbf{x}(t)),
    \]
    and Gronwall's inequality gives the stated exponential bound.
\end{proof}

\begin{lemma}[Window contraction]
    \label{lemma:window_contraction}
    Under Assumption~\ref{ass:uniform_pos}, for any solution $\mathbf{x}(t)$ of the
    scalar consensus flow in Lemma~\ref{lem:instant_contraction} and any $T_w>0$ there exists
    $\rho\in(0,1)$ such that
    \[
        \mathrm{diam}(\mathbf{x}(t+T_w)) \le (1-\rho)\,\mathrm{diam}(\mathbf{x}(t)),\qquad \forall t\ge 0.
    \]
    One can take $\rho \defeq 1-e^{-2\alpha T_w}$.
\end{lemma}

\begin{proof}
    Apply Lemma~\ref{lem:instant_contraction} over the interval $[t,t+T_w]$.
\end{proof}

\begin{theorem}[Consensus trap under time-varying positive mixing]
    \label{thm:time_varying_consensus}
    Consider the time-varying consensus flow
    \[
        \frac{d\mathbf{H}(t)}{dt} = -(\mathbf{I} - \mathbf{P}(t))\mathbf{H}(t),
    \]
    where $\mathbf{P}(t)$ satisfies Assumption~\ref{ass:uniform_pos}. Then for each
    feature dimension $k\in\{1,\dots,m\}$, there exist constants $C,\kappa>0$ such
    that
    \[
        \mathrm{diam}(\mathbf{H}_{\cdot k}(t)) \le C e^{-\kappa t}\,\mathrm{diam}(\mathbf{H}_{\cdot k}(0)).
    \]
    Moreover, there exists a fixed vector $\mathbf{y}_\infty\in\mathbb{R}^m$ such
    that
    \[
        \lim_{t\to\infty}\|\mathbf{H}(t)-\mathbf{1}\mathbf{y}_\infty^\top\|_F = 0.
    \]
\end{theorem}

\begin{proof}
    Fix a feature dimension $k\in\{1,\dots,m\}$ and define $\mathbf{x}(t)\defeq \mathbf{H}_{\cdot k}(t)\in\mathbb{R}^N$.
    From the matrix ODE, $\mathbf{x}(t)$ satisfies
    \[
        \dot{\mathbf{x}}(t) = -(\mathbf{I} - \mathbf{P}(t))\mathbf{x}(t).
    \]
    Applying Lemma~\ref{lem:instant_contraction} yields
    \[
        \mathrm{diam}(\mathbf{H}_{\cdot k}(t))
        \le e^{-2\alpha t}\,\mathrm{diam}(\mathbf{H}_{\cdot k}(0)).
    \]
    Thus we can take $C=1$ and $\kappa=2\alpha$.

    It remains to show that the consensus value is fixed. Let $M(t)\defeq\max_i
        \mathbf{x}_i(t)$ and $m(t)\defeq\min_i \mathbf{x}_i(t)$. Because each row of
    $\mathbf{P}(t)$ is stochastic, $(\mathbf{P}(t)\mathbf{x}(t))_i$ lies in
    $[m(t),M(t)]$ for every $i$. At an index attaining the maximum, $\dot{M}(t)\le
        -M(t)+M(t)=0$ for almost every $t$; at an index attaining the minimum,
    $\dot{m}(t) \ge -m(t)+m(t)=0$. Hence $M(t)$ is non-increasing and bounded
    below, while $m(t)$ is non-decreasing and bounded above. Both have limits, and
    the exponential diameter bound implies $M(t)-m(t)\to 0$. Therefore all entries
    of $\mathbf{x}(t)$ converge to a common scalar $y_{\infty,k}$.

    Applying this argument to all feature dimensions gives
    $\mathbf{y}_\infty=(y_{\infty,1},\dots,y_{\infty,m})$. Since every coordinate
    diameter vanishes and every coordinate has the fixed limit above,
    \[
        \lim_{t\to\infty}\|\mathbf{H}(t)-\mathbf{1}\mathbf{y}_\infty^\top\|_F = 0.
    \]
\end{proof}

\begin{lemma}[Convex-hull invariance]
    \label{lem:convex_hull}
    Let $\mathbf{P}(t)$ be row-stochastic and piecewise continuous on $[0,\infty)$, and let
    $\boldsymbol{\Phi}(t,s)$ be the state-transition matrix of $\dot{\mathbf{X}}=(\mathbf{P}(t)-\mathbf{I})\mathbf{X}$. For all
    $t\ge s\ge0$, $\boldsymbol{\Phi}(t,s)$ is entrywise nonnegative and row-stochastic. Consequently, every solution of
    $\dot{\mathbf{H}}=(\mathbf{P}(t)-\mathbf{I})\mathbf{H}$ satisfies
    $\mathbf{h}_i(t)=\sum_j\boldsymbol{\Phi}_{ij}(t,0)\,\mathbf{h}_j(0)\in\mathrm{conv}\{\mathbf{h}_1(0),\dots,\mathbf{h}_N(0)\}$ and
    $\|\mathbf{h}_i(t)\|\le\max_j\|\mathbf{h}_j(0)\|$ for every norm. If $\mathbf{P}$ is fixed and irreducible, the consensus
    limit $\sum_j\boldsymbol{\pi}_j\mathbf{h}_j(0)$ of Theorem~\ref{thm:monostability} is a point of this hull. The same
    conclusions hold along any solution of a state-dependent flow
    $\dot{\mathbf{H}}=(\mathbf{P}(\mathbf{H}(t),t)-\mathbf{I})\mathbf{H}$, applied to the realized operator.
\end{lemma}

\begin{proof}
    Write $\boldsymbol{\Phi}(t,s)=e^{-(t-s)}\boldsymbol{\Psi}(t,s)$, where $\boldsymbol{\Psi}$ solves
    $\partial_t\boldsymbol{\Psi}(t,s)=\mathbf{P}(t)\boldsymbol{\Psi}(t,s)$ with $\boldsymbol{\Psi}(s,s)=\mathbf{I}$. Because $\mathbf{P}(t)$ is bounded,
    the Peano--Baker series
    \[
        \boldsymbol{\Psi}(t,s)=\mathbf{I}+\int_s^t\mathbf{P}(t_1)\,dt_1+\int_s^t\!\!\int_s^{t_1}\mathbf{P}(t_1)\mathbf{P}(t_2)\,dt_2\,dt_1+\cdots
    \]
    converges, and all its terms are entrywise nonnegative, so $\boldsymbol{\Phi}(t,s)\ge0$. Since
    $(\mathbf{P}(t)-\mathbf{I})\mathbf{1}=\mathbf{0}$, the constant function $\mathbf{1}$ solves
    $\dot{\mathbf{x}}=(\mathbf{P}(t)-\mathbf{I})\mathbf{x}$, hence $\boldsymbol{\Phi}(t,s)\mathbf{1}=\mathbf{1}$ by uniqueness. Each row of
    $\boldsymbol{\Phi}(t,0)$ therefore contains convex weights, which gives the hull statement; the norm bound follows from
    convexity of norms. For fixed irreducible $\mathbf{P}$, the limit $\mathbf{1}\boldsymbol{\pi}^\top\mathbf{H}(0)$ has rows
    $\sum_j\boldsymbol{\pi}_j\mathbf{h}_j(0)$ with $\boldsymbol{\pi}\ge0$ and $\boldsymbol{\pi}^\top\mathbf{1}=1$. For a state-dependent flow, a solution
    defines the piecewise continuous row-stochastic operator $\mathbf{P}(t)\defeq\mathbf{P}(\mathbf{H}(t),t)$, to which the
    linear argument applies.
\end{proof}

\begin{corollary}[Dense soft-attention induces the consensus trap]
    \label{cor:dense_sa}
    Let $\mathbf{P}(t)$ be defined by fully connected softmax attention with temperature $\tau_{\mathrm{attn}}>0$,
    \[
        \mathbf{P}_{ij}(t) \defeq \frac{\exp(\langle \mathbf{h}_i(t),\mathbf{h}_j(t)\rangle/\tau_{\mathrm{attn}})}
        {\sum_{\ell=1}^{N} \exp(\langle \mathbf{h}_i(t),\mathbf{h}_{\ell}(t)\rangle/\tau_{\mathrm{attn}})},
    \]
    and let $\mathbf{H}(t)$ solve $\dot{\mathbf{H}}=(\mathbf{P}(t)-\mathbf{I})\mathbf{H}$. With
    $B\defeq\max_j\|\mathbf{h}_j(0)\|_2$, we have $\|\mathbf{h}_i(t)\|_2\le B$ and
    $\mathbf{P}_{ij}(t)\ge N^{-1}\exp(-2B^2/\tau_{\mathrm{attn}})$ for all $i,j$ and all $t\ge0$, and the dynamics
    converge to a fixed rank-one consensus state.
\end{corollary}

\begin{proof}
    The vector field is locally Lipschitz, so a unique solution exists locally. Along it, $\mathbf{P}(t)$ is continuous
    and row-stochastic, so Lemma~\ref{lem:convex_hull} gives $\|\mathbf{h}_i(t)\|_2\le B$; the solution therefore stays
    bounded and extends to all $t\ge0$. By Cauchy--Schwarz, $|\langle \mathbf{h}_i(t), \mathbf{h}_j(t) \rangle| \le B^2$.
    Thus each numerator is at least $\exp(-B^2/\tau_{\mathrm{attn}})$, while the denominator is at most
    $N\exp(B^2/\tau_{\mathrm{attn}})$. Hence
    \[
        \mathbf{P}_{ij}(t) \ge \frac{1}{N} \exp\left(-\frac{2B^2}{\tau_{\mathrm{attn}}}\right).
    \]
    This satisfies Assumption~\ref{ass:uniform_pos}, so
    Theorem~\ref{thm:time_varying_consensus} applies.
\end{proof}

\subsection{Proof of Proposition~\ref{prop:hysteresis}}
\begin{proof}
    Equilibria satisfy $u_{ij}-u_{ij}^3+\mathcal{F}_{ij}=0$, \ie,~\eqref{eq:cubic_fixed_points}.
    At a saddle-node transition, the cubic has a double root, so both
    $f(u)\defeq u^3-u-\mathcal{F}_{ij}=0$ and $f'(u)=3u^2-1=0$ hold.
    Thus $u=\pm 1/\sqrt{3}$ and $\mathcal{F}_{ij}=u^3-u=\pm 2/(3\sqrt{3})$,
    giving~\eqref{eq:fcrit}. The stated stability and regime characterization follow from standard
    one-dimensional bifurcation analysis of the forced double-well potential.
\end{proof}

\subsection{Double-well as a minimal bistable normal form}
\label{app:double_well_normal_form}

\begin{proposition}[Minimal smooth bistable normal form]
    Consider a scalar edge order parameter $u$ whose two phases should be symmetric in the absence of evidence and stable at nonzero values. Among bounded-below even polynomial potentials, the lowest-order family that can represent two stable minima separated by an unstable barrier is
    \[
        V_0(u)=\frac{\kappa}{4}u^4-\frac{b}{2}u^2,\qquad \kappa,b>0.
    \]
    Under an additive force $F$, the gradient flow $\dot u=bu-\kappa u^3+F$ becomes, after the rescaling
    $\bar u=\sqrt{\kappa/b}\,u$ and $\bar t=bt$,
    \[
        \frac{d\bar u}{d\bar t}=\bar u-\bar u^3+\bar F,\qquad \bar F\defeq\frac{\sqrt{\kappa}}{b^{3/2}}\,F .
    \]
    Thus Eq.~\eqref{eq:bistable_dynamics} is the canonical low-order smooth model for
    two persistent edge phases with an analyzable switching threshold.
\end{proposition}

\begin{proof}
    A nonconstant even quadratic potential has at most one stable critical point when bounded below, or is unbounded below otherwise, and therefore cannot encode two stable phases. Adding the quartic term gives the first bounded-below even polynomial with stationary points at $u=0$ and $u=\pm\sqrt{b/\kappa}$; the outer points are stable minima and the origin is the separating barrier. An additive force contributes a linear tilt $-Fu$, which preserves differentiability and yields the gradient flow $\dot u=bu-\kappa u^3+F$. Substituting $u=\sqrt{b/\kappa}\,\bar u$ and $t=\bar t/b$ gives $b\sqrt{b/\kappa}\,d\bar u/d\bar t=b\sqrt{b/\kappa}\,(\bar u-\bar u^3)+F$, which is the stated normal form.
\end{proof}

\paragraph{The HGODE scaling and the role of $\lambda$.}
Each scalar topology equation in Eq.~\eqref{eq:hgode_system} reads $\tau_{topo}\dot u=au-u^3+\mathcal{F}$ with
$a\defeq1-\lambda\in(0,1]$, \ie, $\kappa=1$ and $b=a$ up to the time constant $\tau_{topo}$. The substitution
$\bar u=u/\sqrt a$, $\bar t=at/\tau_{topo}$, $\bar{\mathcal{F}}=\mathcal{F}/a^{3/2}$ gives
$d\bar u/d\bar t=\bar u-\bar u^3+\bar{\mathcal{F}}$, so Proposition~\ref{prop:hysteresis} yields saddle-node transitions at
$|\mathcal{F}|=a^{3/2}\cdot2/(3\sqrt3)=\mathcal{F}_{\mathrm{crit}}(\lambda)$, as in Eq.~\eqref{eq:fcrit_lambda}. For $\mathcal{F}=0$, the
wells $u=\pm\sqrt a$ have energy $-a^2/4$ and the barrier at $u=0$ has energy $0$, so the barrier height is $a^2/4$.
Increasing $\lambda$ therefore moves the wells toward zero and lowers both the barrier and the switching threshold:
phase changes become easier, while hysteretic memory weakens. The force parameterization (\eg, the MLP in
Eq.~\eqref{eq:force_def}) does not enter $\mathcal{F}_{\mathrm{crit}}(\lambda)$; it only determines whether and when the realized
force crosses it.

\subsection{Force margin implies polarization of edge potentials}
\label{app:polarization}

\begin{lemma}[Force margin implies global polarization]
    \label{lem:margin_polarization}
    Consider the scalar edge-potential dynamics
    \begin{equation}
        \label{eq:scalar_hysteresis_app}
        \dot u = a u-u^3+F,
    \end{equation}
    where $a=1-\lambda>0$ and $F\in\mathbb{R}$ is constant. Define
    $F_{\mathrm{crit}}(\lambda)\defeq 2(1-\lambda)^{3/2}/(3\sqrt{3})$.
    If $|F|>F_{\mathrm{crit}}(\lambda)$, then \eqref{eq:scalar_hysteresis_app} admits a \emph{unique} equilibrium $u^\star(F)$,
    and this equilibrium is \emph{globally asymptotically stable}. Moreover,
    $\mathrm{sign}(u^\star(F))=\mathrm{sign}(F)$.
    In particular, if $|F|\ge F_{\mathrm{crit}}(\lambda)+\delta$ for some $\delta>0$, then every trajectory satisfies
    $u(t)\to u^\star(F)$ as $t\to\infty$, i.e., the potential polarizes to the connected well for $F>0$
    and to the insulated well for $F<0$.
\end{lemma}

\begin{proof}
    Define $g(u)\defeq au-u^3+F$. Equilibria are roots of $g(u)=0$, i.e., of the cubic $u^3-au=F$.
    At a saddle-node transition, both $u^3-au-F=0$ and $3u^2-a=0$ hold, giving
    $u=\pm\sqrt{a/3}$ and the critical magnitude $2a^{3/2}/(3\sqrt{3})=F_{\mathrm{crit}}(\lambda)$.
    Hence, when $|F|>F_{\mathrm{crit}}(\lambda)$, the cubic has exactly one real root; denote it by $u^\star(F)$.

    Local stability follows from $g'(u)=a-3u^2$. In the monostable regime, the
    unique equilibrium is stable, so $g'(u^\star)<0$. Since $g$ is continuous and
    has a unique zero, $g(u)$ has a fixed sign on each side of $u^\star$: $g(u)<0$
    for all $u>u^\star$ and $g(u)>0$ for all $u<u^\star$. Thus trajectories are
    monotone toward $u^\star$ and bounded, so they converge to the unique
    equilibrium. Finally, $g(0)=F$ and the cubic tails imply
    $\mathrm{sign}(u^\star(F))=\mathrm{sign}(F)$.
\end{proof}

\begin{lemma}[Sustained time-varying force margin]
    \label{lem:timevarying_margin_polarization}
    Consider the scalar edge-potential dynamics
    \[
        \dot u = a u-u^3+F(t),
    \]
    where $a=1-\lambda>0$ and $F(t)$ is bounded and piecewise continuous. Let
    $c=\sqrt{a/3}$, and call $\{u\ge c\}$ the positive phase and $\{u\le -c\}$ the
    negative phase. If there exist $\delta>0$ and $T_0\ge 0$ such that $F(t)\ge
        F_{\mathrm{crit}}(\lambda)+\delta$ for all $t\ge T_0$, then every trajectory
    enters the positive phase in finite time and remains there. If $F(t)\le
        -F_{\mathrm{crit}}(\lambda)-\delta$ for all $t\ge T_0$, then every trajectory
    enters the negative phase in finite time and remains there. If additionally
    $F(t)\to F_\infty$, then $u(t)\to u^\star(F_\infty)$, the unique equilibrium of
    $a u-u^3+F_\infty=0$.
\end{lemma}

\begin{proof}
    Let $g_t(u)\defeq au-u^3+F(t)$. The polynomial $au-u^3$ has local minimum $-F_{\mathrm{crit}}(\lambda)$ at $u=-c$ and local maximum $F_{\mathrm{crit}}(\lambda)$ at $u=c$. If $F(t)\ge F_{\mathrm{crit}}(\lambda)+\delta$, then $g_t(u)\ge\delta$ for all $u\le c$. Hence any trajectory below $c$ increases at speed at least $\delta$ until it reaches $c$, and at $u=c$ the vector field is positive, so the trajectory cannot cross back into $u<c$. This proves finite-time entry and persistence of the positive phase. The negative-force case is symmetric: if $F(t)\le -F_{\mathrm{crit}}(\lambda)-\delta$, then $g_t(u)\le -\delta$ for all $u\ge -c$, giving finite-time entry into and persistence in the negative phase.

    The cubic term $-u^3$ makes all trajectories bounded. When $F(t)\to F_\infty$,
    the limiting autonomous scalar equation is in the monostable regime and has a
    unique globally stable equilibrium $u^\star(F_\infty)$ by
    Lemma~\ref{lem:margin_polarization}. Standard one-dimensional comparison with
    the autonomous equations using forces $F_\infty\pm\varepsilon$ traps the limsup
    and liminf of $u(t)$ between roots that converge to $u^\star(F_\infty)$ as
    $\varepsilon\downarrow0$. Thus $u(t)\to u^\star(F_\infty)$.
\end{proof}

\section{Supplementary Experiments and Implementation Details}
\label{app:supplementary_details}

This appendix provides the full synthetic protocol, graph-transformer
comparisons, efficiency profiling, and reproducibility details referenced from
Sec.~\ref{sec:realworld}. The main text focuses on the core mechanism and the
closest continuous-depth baselines.

\subsection{Implementation architecture and algorithm}
\label{app:architecture_algorithm}
This subsection specifies the released implementation. Let $d$ be the hidden width, $p$ the dropout rate, $\mathbf{x}_i$ the input features of node $i$, and $\mathrm{Drop}_p$ dropout with rate $p$. Message passing follows the convention of Eq.~\eqref{eq:implemented_gate}: an ordered candidate $(i,j)$ carries a message from sender $j$ to receiver $i$.

\paragraph{Encoder and initial state.}
The encoder is
\[
    \mathbf{h}^{enc}_i=\tanh\!\Big(\mathrm{LN}\big(\mathbf{W}_2\,\mathrm{Drop}_p\big(\mathrm{GELU}(\mathrm{LN}(\mathbf{W}_1\mathbf{x}_i+\mathbf{b}_1))\big)+\mathbf{b}_2\big)\Big).
\]
An optional positional encoding is added after the bounded activation, giving $\mathbf{h}^{enc}_i+\mathrm{PE}_i$. The default degree encoding is $\mathrm{PE}_i=\mathrm{GELU}\big(\mathrm{LN}(\mathbf{W}_{pe}[\mathbf{E}_{\min(\deg_i,100)}\Vert(\mathbf{w}\log(1+\deg_i)+\mathbf{b})]+\mathbf{b}_{pe})\big)$, where $\mathbf{E}$ is a learned embedding table and $\deg_i$ counts the arcs received by node $i$; Laplacian and random-walk encodings are available alternatives. With the optional virtual node, a learned vector $\mathbf{v}^{(0)}$, shared by all graphs, is broadcast to the nodes of each graph $g$,
\[
    \mathbf{h}_i\leftarrow\mathbf{h}_i+\sigma(a)\,\mathrm{MLP}_{b}(\mathbf{v}_{g(i)}),\qquad
    \mathbf{v}_g\leftarrow\mathbf{v}_g+\mathrm{MLP}_{u}\Big(\tfrac{1}{|g|}\textstyle\sum_{i\in g}\mathbf{h}_i\Big),
\]
with a learned scalar $a$ (initialized at $-2$) and $\mathrm{MLP}_{b}=$ Linear--LN--GELU; $\mathrm{MLP}_{u}$ is Linear($d\!\to\!2d$)--LN--GELU--$\mathrm{Drop}_{0.1}$--Linear($2d\!\to\!d$)--LN, and the mean uses the already broadcast node states. The resulting node states form $\mathbf{H}(0)$. The virtual node is not part of the ODE state.

\paragraph{ODE, force network, and solver.}
Eq.~\eqref{eq:hgode_system} is integrated from $(\mathbf{H}(0),\mathbf{U}(0))$, with $\mathcal{G}_{\phi}$ given by Eqs.~\eqref{eq:implemented_gate}--\eqref{eq:implemented_operator} and $\mathbf{U}(0)$ by Algorithm~\ref{alg:candidate_pool}. The force network maps the receiver-then-sender concatenation $[\mathbf{h}_i\Vert\mathbf{h}_j]$ through Linear($2d\!\to\!64$)--ReLU--Linear($64\!\to\!32$)--ReLU--Linear($32\!\to\!1$) before the bounded output $s\tanh(\cdot)$ of Eq.~\eqref{eq:force_def}, so $\mathcal{F}_{ij}\neq\mathcal{F}_{ji}$ in general. We use \texttt{torchdiffeq} with \texttt{dopri5}, $\texttt{rtol}=\texttt{atol}=10^{-5}$, and at most 1000 solver steps; if adaptive integration fails, the implementation falls back to a fixed-step RK4 solve with 8 steps on $[0,T]$. Non-finite entries of the terminal state are replaced by finite values before the readout.

\paragraph{Readout.}
With a virtual node, the updated $\mathbf{v}_g$ is broadcast once more into $\mathbf{H}(T)$ as above. The post-ODE map is $\mathbf{h}^{post}_i=\mathrm{Drop}_{p/2}\big(\mathrm{GELU}(\mathrm{LN}(\mathbf{W}_{post}\mathbf{h}_i(T)+\mathbf{b}_{post}))\big)$. The node head fuses pre- and post-ODE states and adds a residual,
\[
    \mathbf{f}_i=\mathrm{Drop}_{p/2}\Big(\mathrm{GELU}\big(\mathrm{LN}(\mathbf{W}_f[\mathbf{h}_i(0)\Vert\mathbf{h}^{post}_i]+\mathbf{b}_f)\big)\Big),\qquad
    \hat{\mathbf{y}}_i=\mathrm{MLP}_c\big(\mathbf{f}_i+\mathbf{h}^{post}_i\big),
\]
and the graph head computes $\hat{\mathbf{y}}_g=\mathrm{MLP}_c\big(\mathrm{pool}_{i\in g}\,\mathbf{h}^{post}_i\big)$ with mean, max, sum, or concatenated mean/max pooling. In both heads, $\mathrm{MLP}_c$ is Linear--LN--GELU--$\mathrm{Drop}_p$--Linear. The graph head therefore has no pre-/post-ODE fusion or residual.

\paragraph{Edge attributes.}
When edge attributes $\mathbf{a}_{ij}$ are available, an encoder $e_{ij}=\sigma\big(\mathbf{w}_e^\top\mathrm{GELU}(\mathbf{W}_e\mathbf{a}_{ij}+\mathbf{b}_e)+c_e\big)\in(0,1)$ rescales the initial potential of each observed edge to $u_{ij}(0)=u_{obs}\,e_{ij}$. Edge attributes enter no other component.

\paragraph{Training and evaluation.}
Algorithm~\ref{alg:hgode_step} summarizes one training step. Validation and test passes run lines~1--4 with the same horizon $T$ and use no labels, margin pairs, or pseudo-partitions.

\begin{algorithm}[t]
    \caption{One HGODE training step}
    \label{alg:hgode_step}
    \begin{algorithmic}[1]
        \Require graph(s) with arcs $\mathcal{E}$ and features $\mathbf{X}$; training labels; hyperparameters (Tables~\ref{tab:hparam_roles}--\ref{tab:hparam_starts})
        \State Build $\mathcal{E}_{cand}$ and $\mathbf{U}(0)$ from $\mathcal{E}$ alone (Algorithm~\ref{alg:candidate_pool})
        \State $\mathbf{H}(0)\gets$ encoder, positional encoding, and virtual-node broadcast applied to $\mathbf{X}$
        \State $(\mathbf{H}(T),\mathbf{U}(T))\gets$ integrate Eq.~\eqref{eq:hgode_system} on $[0,T]$, evaluating Eqs.~\eqref{eq:implemented_gate}--\eqref{eq:implemented_operator} at every solver call
        \State Compute predictions with the node or graph head from $\mathbf{H}(0)$ and $\mathbf{H}(T)$
        \State Form $\mathcal{P},\mathcal{N}\subseteq\mathcal{E}_{cand}$ from training labels, or from Algorithm~\ref{alg:pseudo_partition} when no label partition is supplied (Appendix~\ref{app:multi_label_margin})
        \State Evaluate $\mathcal{L}_{margin}$ (Eq.~\eqref{eq:force_contra}) with $\mathcal{F}_\theta$ applied to $\mathbf{H}(0)$ on sampled pairs
        \State Update all parameters with $\nabla(\mathcal{L}_{task}+\beta\mathcal{L}_{margin})$ and gradient-norm clipping at $1.0$
    \end{algorithmic}
\end{algorithm}

\subsection{Candidate-pool construction}
\label{app:candidate_construction}
Algorithm~\ref{alg:candidate_pool} builds the candidate pool of one graph; batched graphs are processed one graph at a time, so graph-level tasks build a pool within each graph, including unseen test graphs. The construction uses only the input graph, never labels or split masks, and it is deterministic (seeded random proposals, canonicalized eigenbases, and quantized similarities), so it is rebuilt identically at every forward pass and stays fixed during each ODE solve. Observed arcs are kept as supplied, so directed inputs remain directed, whereas proposals are inserted in both directions. The LapPE ranking does not exclude observed neighbors; such selections coincide with an observed arc and keep its potential after deduplication. A proposal type is skipped for graphs larger than its size cap, so the realized pool can be smaller than the nominal budget $(k_{2hop},k_{lap},r_{rand})$.

\begin{algorithm}[t]
    \caption{Static candidate pool for one graph with $n$ nodes}
    \label{alg:candidate_pool}
    \begin{algorithmic}[1]
        \Require observed arcs $\mathcal{E}$; budgets $k_{2hop},k_{lap},r_{rand}$; cap $c_{rand}$; LapPE dimension $d_{PE}$; size caps $n^{\max}_{2hop},n^{\max}_{lap}$; potentials $u_{obs}>0>u_{new}$
        \State $\mathbf{A}^{sym}\gets$ 0/1 adjacency of the symmetrization of $\mathcal{E}$, without self-loops
        \If{$k_{2hop}>0$ and $n\le n^{\max}_{2hop}$}
        \State For each $i$, propose the $k_{2hop}$ non-neighbors $j\neq i$ with the most common neighbors, keeping only $(\mathbf{A}^{sym}\mathbf{A}^{sym})_{ij}>0$
        \EndIf
        \If{$k_{lap}>0$ and $n\le n^{\max}_{lap}$}
        \State Eigendecompose $\mathbf{I}-\mathbf{D}^{-1/2}\mathbf{A}^{sym}\mathbf{D}^{-1/2}$ in float64 and canonicalize degenerate eigenspaces
        \State $\mathbf{z}_i\gets$ row $i$ of eigenvectors $2,\dots,d_{PE}+1$, $\ell_2$-normalized
        \State For each $i$, propose the $k_{lap}$ nodes $j\neq i$ with the largest $\mathbf{z}_i^\top\mathbf{z}_j$, rounded to $10^{-6}$ with ties broken by node index
        \EndIf
        \If{$r_{rand}>0$}
        \State For each $i$, propose $\min(\lceil r_{rand}n\rceil,c_{rand})$ nodes $j\neq i$ drawn uniformly, with a generator seeded by a hash of $\mathcal{E}$
        \EndIf
        \State Insert every proposal in both directions; $\mathcal{E}_{cand}\gets\mathcal{E}\cup\{\text{proposals}\}$ without duplicates, observed arcs taking precedence
        \State $u_{ij}(0)\gets u_{obs}$ (times $e_{ij}$ with edge attributes) for observed arcs and $u_{new}$ for the remaining candidates
        \State \Return $\mathcal{E}_{cand}$, $\mathbf{U}(0)$
    \end{algorithmic}
\end{algorithm}

\subsection{Force-margin pair construction}
\label{app:multi_label_margin}
\paragraph{Sampling.}
At every training step, the compatible and incompatible candidate pairs $\mathcal{P},\mathcal{N}\subseteq\mathcal{E}_{cand}$ are formed as in Table~\ref{tab:pair_construction}, and at most $n_{pair}$ pairs of each type are drawn uniformly without replacement to give $\mathcal{P}_s$ and $\mathcal{N}_s$ in Eq.~\eqref{eq:force_contra}. For single-label node tasks, the partition is given by the training labels, and every pair with an endpoint outside the training set is excluded.

\paragraph{Multi-label tasks.}
For multi-label node datasets such as ogbn-proteins, each training-labeled node has a binary label vector $\mathbf{y}_i\in\{0,1\}^{C}$. For candidate pairs whose endpoints both have training labels, we compute
\[
    c_{ij}=\frac{\langle \mathbf{y}_i,\mathbf{y}_j\rangle}{\|\mathbf{y}_i\|_2\|\mathbf{y}_j\|_2+10^{-12}}.
\]
Given a threshold $\rho_{ml}$, pairs with $c_{ij}\ge\rho_{ml}$ are compatible
and enter $\mathcal P$, while pairs with $c_{ij}<\rho_{ml}$ enter $\mathcal N$;
pairs touching validation/test nodes are omitted from the supervised margin
term or assigned only by the label-free pseudo-partition. The threshold is
selected using the training/validation protocol only.

\paragraph{Label-free pseudo-partition.}
When no label partition is supplied, which is the case for graph-level tasks,
Algorithm~\ref{alg:pseudo_partition} partitions each training graph by $k$-means
on its detached pre-ODE states. It runs without gradients and is never used at
validation or test time. In the released implementation it is recomputed
whenever the margin loss is evaluated, from the pre-ODE states of the current
training forward pass, which include encoder dropout.

\begin{algorithm}[t]
    \caption{Label-free pseudo-partition of one training graph}
    \label{alg:pseudo_partition}
    \begin{algorithmic}[1]
        \Require detached pre-ODE states $\mathbf{h}_1(0),\dots,\mathbf{h}_n(0)$ of the graph; cluster count $K_{cfg}$
        \State $K\gets\min(K_{cfg},n)$ (a single cluster if $K=1$) and $\mathbf{z}_i\gets\mathbf{h}_i(0)/\|\mathbf{h}_i(0)\|_2$
        \State $\mathbf{m}_k\gets\mathbf{z}_{\iota_k}$ with $\iota_k=\mathrm{round}\big((k-1)(n-1)/(K-1)\big)$, $k=1,\dots,K$
        \For{$6$ iterations}
        \State $c_i\gets\arg\min_k\|\mathbf{z}_i-\mathbf{m}_k\|_2$ for every node $i$
        \State $\mathbf{m}_k\gets$ $\ell_2$-normalized mean of $\{\mathbf{z}_i:c_i=k\}$, unchanged if the cluster is empty
        \EndFor
        \State \Return the partition $\{c_i\}$
    \end{algorithmic}
\end{algorithm}

\begin{table}[H]
    \caption{Centralized force-margin pair construction. Pair labels are used only for the auxiliary margin loss; the task loss remains the primary supervision signal.}
    \label{tab:pair_construction}
    \centering
    {\scriptsize
        \setlength{\tabcolsep}{3pt}
        \renewcommand{\arraystretch}{1.08}
        \begin{tabular}{p{0.23\textwidth}p{0.27\textwidth}p{0.27\textwidth}p{0.17\textwidth}}
            \toprule
            Regime & Positive pairs $\mathcal P$                                          & Negative pairs $\mathcal N$ & Labels used \\
            \midrule
            Single-label node classification
                   & Same training label among candidate pairs
                   & Different training labels among candidate pairs
                   & Training labels only                                                                                             \\
            Multi-label node classification
                   & Training-vector cosine $c_{ij}\ge\rho_{ml}$
                   & Training-vector cosine $c_{ij}<\rho_{ml}$
                   & Training label vectors only                                                                                      \\
            Unlabeled, validation, or test nodes
                   & Label-free pseudo-partition pairs, or omitted from supervised margin
                   & Label-free pseudo-partition pairs, or omitted from supervised margin
                   & No validation/test labels                                                                                        \\
            Graph-level tasks
                   & Same within-graph pseudo-cluster on training graphs
                   & Different within-graph pseudo-clusters on training graphs
                   & No graph labels for node pairs                                                                                   \\
            \bottomrule
        \end{tabular}}
\end{table}

\subsection{Synthetic diagnostic protocol and perturbation tests}
\label{app:synthetic_protocol}

\paragraph{SBM generation.}
We generate graphs from a $K$-block SBM with node partition $\{C_1,\dots,C_K\}$
and block sizes $\{n_k\}_{k=1}^K$ ($\sum_k n_k=N$). For $i\in C_k$ and $j\in
    C_\ell$, edges are sampled independently as
\begin{equation*}
    \mathbf{A}^{(0)}_{ij} \sim \mathrm{Bernoulli}(p_{k\ell}),
    \qquad
    p_{k\ell}=
    \begin{cases}
        p_{\mathrm{in}},  & k=\ell,     \\
        p_{\mathrm{out}}, & k\neq \ell,
    \end{cases}
\end{equation*}
with $p_{\mathrm{in}} > p_{\mathrm{out}}$ unless stated otherwise. Node features are initialized as
\begin{equation*}
    \mathbf{h}_i(0)=\mu_{c(i)}+\epsilon_i,
    \qquad \epsilon_i \sim \mathcal{N}(0,\sigma^2 \mathbf{I}),
\end{equation*}
where $c(i)\in\{1,\dots,K\}$ denotes the SBM block label.

\paragraph{Synthetic baselines and diagnostics.}
The soft-attention Graph ODE baseline constructs a strictly positive,
row-stochastic transition matrix
\begin{equation*}
    \mathbf{P}_{ij}(t) = \frac{\exp\!\left(s_{ij}(t)/\tau_{\mathrm{attn}}\right)}
    {\sum_{k=1}^N \exp\!\left(s_{ik}(t)/\tau_{\mathrm{attn}}\right)},
    \qquad
    s_{ij}(t)=\langle \mathbf{h}_i(t), \mathbf{h}_j(t)\rangle,
\end{equation*}
and evolves features via
\begin{equation*}
    \frac{d\mathbf{H}(t)}{dt} = -\big(\mathbf{I}-\mathbf{P}(t)\big)\mathbf{H}(t).
\end{equation*}
HGODE follows Eq.~\ref{eq:hgode_system}, with effective adjacency recovered by
$\mathbf{A}_{ij}(t)=\sigma(\mathbf{U}_{ij}(t)/\tau)$. We validate
Theorem~\ref{thm:monostability} by sweeping $\tau_{\mathrm{attn}}$ and tracking
(i) representation mixing via intra- and inter-cluster distances, (ii)
clustering quality via the silhouette score, and (iii) HGODE's mean edge
potential $\mathbf{U}_{ij}(t)$ for intra- and inter-cluster pairs.

\paragraph{Synthetic perturbation tests.}
To evaluate robustness under feature corruption, we generate SBM graphs with
increased cross-cluster connectivity by setting $p_{\mathrm{out}}=0.3$ while
keeping other SBM parameters fixed. We sweep the noise standard deviation
$\sigma$ in the Gaussian feature initialization above, use 2000 graphs in total
with an 80/20 train/validation split, and compare a DeepSets-style MLP
baseline, the soft-attention Graph ODE baseline, and HGODE. For each model and
noise level, we run 10 random seeds and train for 10 epochs, reporting
validation accuracy as mean with shaded uncertainty.

\begin{figure}[H]
    \centering
    \includegraphics[width=\linewidth]{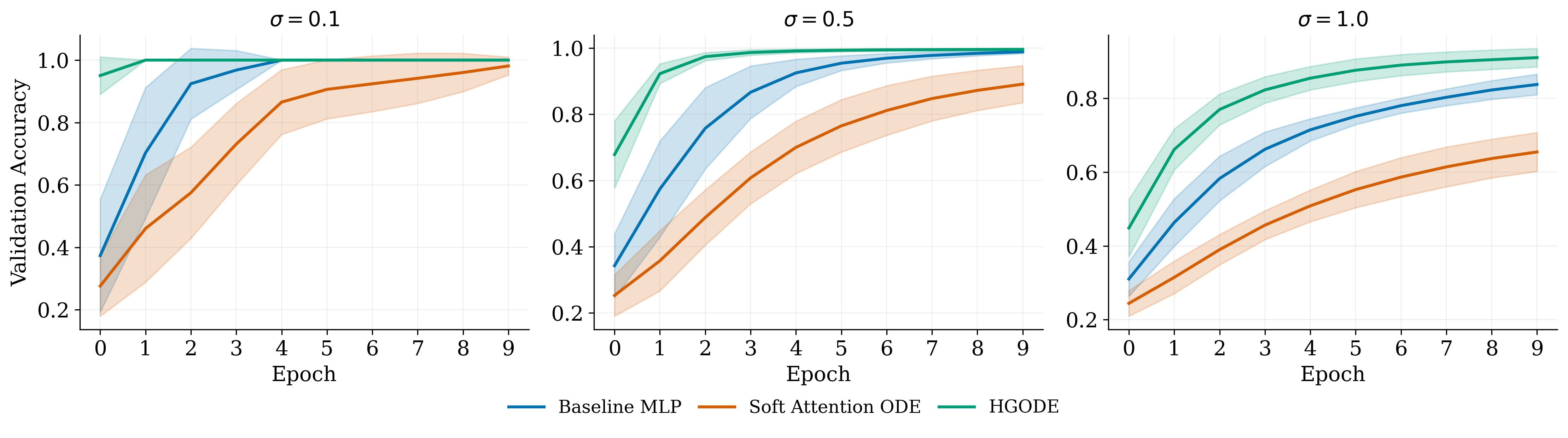}
    \caption{Validation accuracy under feature perturbations on SBM graphs ($\mu=0.5$, $p_{\mathrm{out}}=0.3$). Subplots correspond to noise levels \textbf{Left}: $\sigma=0.1$, \textbf{Middle}: $\sigma=0.5$, \textbf{Right}: $\sigma=1.0$. Solid lines denote mean over 10 seeds and shaded regions indicate $\pm 1$ standard deviation.}
    \label{fig:perturbation_exp}
\end{figure}

Figure~\ref{fig:perturbation_exp} shows that the soft-attention Graph ODE
becomes less robust as $\sigma$ increases, consistent with increased
representation mixing under diffuse global attention. HGODE remains
comparatively stable, indicating that dynamic topology can suppress spurious
feature diffusion and recover informative signal even in low signal-to-noise
regimes.

\subsection{Graph transformer baselines}
\label{app:transformer_baselines}

Graph transformers use high-capacity attention or subgraph-tokenization
mechanisms that are complementary to the continuous-time dynamics studied in
the main text. We choose four recent strong graph transformers, \ie, GraphGPS~\cite{rampasek2022GPS}, Polynormer~\cite{deng2024polynormer}, SGFormer~\cite{wu2023sgformer}, and Subgraphormer~\cite{bar-shalom2024subgraphormer}, and compare them with HGODE. We report them separately to keep Table~\ref{tb:performance}
focused on message-passing and continuous-depth baselines.

\begin{table}[tb]
    \caption{Comparison with graph-transformer baselines. Entries are mean $\pm$ standard deviation under matched splits and comparable parameter budgets when applicable. Roman-empire uses the ten fixed splits of \citet{platonov2022critical}; no protocol-matched node-classification implementation of Subgraphormer was available. The top three values are marked using the same convention as Table~\ref{tb:performance}.}
    \label{tab:transformer_baselines}
    \centering
    {\scriptsize
        \setlength{\tabcolsep}{3pt}
        \renewcommand{\arraystretch}{1.05}
        \begin{tabular}{lccccccc}
            \toprule
            Model         & Cora                    & Chameleon               & Roman-empire            & ogbn-proteins           & ZINC                     & Peptides-func            & ogbg-molpcba             \\
                          & Acc.$\uparrow$          & Acc.$\uparrow$          & Acc.$\uparrow$          & ROC-AUC$\uparrow$       & MAE$\downarrow$          & A.P.$\uparrow$           & A.P.$\uparrow$           \\
            \midrule
            GraphGPS      & 82.95$\pm$1.23          & 41.04$\pm$3.87          & \third{83.38$\pm$0.76}  & 77.25$\pm$0.42          & \best{0.069$\pm$0.008}   & \third{0.655$\pm$0.004}  & \second{0.291$\pm$0.001} \\
            Polynormer    & \third{83.24$\pm$0.72}  & \third{41.69$\pm$3.22}  & \second{92.06$\pm$0.62} & \third{79.41$\pm$0.56}  & --                       & --                       & --                       \\
            SGFormer      & \second{84.59$\pm$0.43} & \second{45.77$\pm$3.48} & 81.47$\pm$0.59          & \second{79.84$\pm$0.61} & --                       & --                       & --                       \\
            Subgraphormer & --                      & --                      & --                      & --                      & \third{0.082$\pm$0.04}   & \second{0.661$\pm$0.005} & \best{0.293$\pm$0.004}   \\
            HGODE         & \best{86.26$\pm$0.78}   & \best{72.56$\pm$1.24}   & \best{92.56$\pm$0.43}   & \best{81.24$\pm$0.63}   & \second{0.078$\pm$0.025} & \best{0.714$\pm$0.022}   & \third{0.278$\pm$0.003}  \\
            \bottomrule
        \end{tabular}}
\end{table}

\subsection{Learned-topology diagnostics and rewiring comparisons}
\label{app:topology_diagnostics}
To inspect the operators that HGODE actually learns, we compare them with two static topology interventions on four node-classification datasets: curvature-based rewiring with BORF~\cite{nguyen2023borf} and diffusion-based graph construction with GDC/PPR~\cite{gasteiger2019diffusion}. CiteSeer uses its public Planetoid split~\cite{yang2016revisiting}; Chameleon~\cite{rozemberczki2021multi}, Texas, and Wisconsin use Geom-GCN split columns 0--2~\cite{pei2020geom}, paired with seeds 42, 1042, and 2042. This Chameleon protocol differs from the duplicate-filtered version used in Table~\ref{tb:performance}. Node features are normalized and input graphs are symmetrized.

\paragraph{Protocol.}
Input, BORF, and GDC/PPR share the same two-layer GCN (16 hidden units, dropout 0.5, learning rate 0.01, weight decay $5\times10^{-4}$, 200 epochs), so their task differences isolate the static topology change. HGODE uses its own backbone (two 256-dimensional message blocks, dropout 0.5, \texttt{dopri5} on $[0,1]$ with $\texttt{rtol}=\texttt{atol}=10^{-5}$, at most 120 epochs with patience 20, learning rate $5\times10^{-4}$ on CiteSeer and $10^{-4}$ elsewhere, weight decay $5\times10^{-4}$), so its task values are descriptive rather than a rewiring-only comparison. All three interventions are selected with validation accuracy and label-free topology guards only, and the selections are frozen before any test or label-dependent quantity is computed. Table~\ref{tab:rewiring_settings} lists the selected settings and the resulting support changes.

\paragraph{Spectral diagnostics.}
For each operator we remove self-loops, symmetrize the weights as $\mathbf{W}=(\mathbf{A}+\mathbf{A}^\top)/2$, and form $\mathbf{L}_{sym}=\mathbf{I}-\mathbf{D}^{-1/2}\mathbf{W}\mathbf{D}^{-1/2}$. We report the algebraic connectivity $\lambda_2(\mathbf{L}_{sym})$ and the number $N_{\lambda<10^{-2}}$ of nontrivial eigenvalues below $10^{-2}$, both on the largest connected component, and the off-block mass, \ie, the fraction of symmetrized weight that crosses a deterministic, graph-only $K$-way spectral partition with $K$ equal to the number of classes. Smaller $\lambda_2$, more near-zero modes, and lower off-block mass indicate stronger near-decomposability; we use them as mechanism diagnostics, not as objectives. For HGODE, the spectra use the final soft operator $\mathbf{A}(T)$ of Eq.~\eqref{eq:A_from_U}, one per split, while the support counts in Table~\ref{tab:rewiring_settings} use the thresholded projection of Eq.~\eqref{eq:evaluation_support_projection}.

\begin{table}[tb]
    \caption{Learned-topology diagnostics. Accuracy and macro-F1 are mean $\pm$ standard deviation over three runs. Spectral quantities are computed on the frozen operators; for HGODE they are mean $\pm$ standard deviation over the three split-specific operators. Input, BORF, and GDC/PPR share the same two-layer GCN, whereas HGODE uses its own backbone. Arrows indicate stronger near-decomposability, not task desirability.}
    \label{tab:topology_diagnostics}
    \centering
    {\scriptsize
        \setlength{\tabcolsep}{4pt}
        \renewcommand{\arraystretch}{1.05}
        \begin{tabular}{llccccc}
            \toprule
            Dataset                 & Method      & Acc. (\%)$\uparrow$    & Macro-F1 (\%)$\uparrow$ & $\lambda_2(\mathbf{L}_{sym})\downarrow$ & $N_{\lambda<10^{-2}}\uparrow$ & Off-block mass$\downarrow$ \\
            \midrule
            CiteSeer                & Input+GCN   & 70.43$\pm$0.86         & 66.22$\pm$0.46          & 0.00155                                 & 9                             & 0.0212                     \\
                                    & BORF+GCN    & 70.60$\pm$0.61         & 66.38$\pm$0.53          & 0.00169                                 & 7                             & 0.0233                     \\
                                    & GDC/PPR+GCN & 70.57$\pm$0.81         & 66.97$\pm$1.32          & 0.00249                                 & 6                             & 0.0274                     \\
                                    & HGODE       & 73.87$\pm$1.10         & 68.22$\pm$1.27          & 0.00169$\pm$0.00023                     & 7.7$\pm$1.2                   & 0.0160$\pm$0.0045          \\
            \midrule
            Chameleon (Geom-GCN)    & Input+GCN   & 56.29$\pm$4.82         & 55.29$\pm$4.42          & 0.00640                                 & 2                             & 0.0089                     \\
                                    & BORF+GCN    & 55.70$\pm$4.96         & 54.84$\pm$4.88          & 0.01362                                 & 0                             & 0.0326                     \\
                                    & GDC/PPR+GCN & 55.77$\pm$1.46         & 55.22$\pm$1.77          & 0.00165                                 & 9                             & 0.0130                     \\
                                    & HGODE       & 55.19$\pm$0.33         & 54.97$\pm$0.39          & 0.00632$\pm$0.00023                     & 2.0$\pm$0.0                   & 0.0102$\pm$0.0010          \\
            \midrule
            Texas                   & Input+GCN   & 58.56$\pm$6.80         & 18.70$\pm$0.95          & 0.06323                                 & 0                             & 0.2509                     \\
                                    & BORF+GCN    & 58.56$\pm$8.69         & 19.33$\pm$6.07          & 0.06320                                 & 0                             & 0.2500                     \\
                                    & GDC/PPR+GCN & 58.56$\pm$4.13         & 22.68$\pm$2.20          & 0.02429                                 & 0                             & 0.0643                     \\
                                    & HGODE       & 76.58$\pm$8.26         & 53.16$\pm$10.55         & 0.10671$\pm$0.01245                     & 0.0$\pm$0.0                   & 0.2035$\pm$0.0608          \\
            \midrule
            Wisconsin               & Input+GCN   & 59.48$\pm$9.67         & 25.47$\pm$3.29          & 0.07652                                 & 0                             & 0.2000                     \\
                                    & BORF+GCN    & 63.40$\pm$6.89         & 27.79$\pm$1.12          & 0.11401                                 & 0                             & 0.2589                     \\
                                    & GDC/PPR+GCN & 65.36$\pm$6.30         & 28.96$\pm$1.49          & 0.03750                                 & 0                             & 0.0891                     \\
                                    & HGODE       & 79.08$\pm$6.89         & 53.94$\pm$9.84          & 0.10289$\pm$0.01113                     & 0.0$\pm$0.0                   & 0.1773$\pm$0.0238          \\
            \bottomrule
        \end{tabular}}
\end{table}

\begin{table}[tb]
    \caption{Selected settings and support changes of the topology interventions in Table~\ref{tab:topology_diagnostics}. The input graphs have 9{,}104 (CiteSeer), 62{,}742 (Chameleon), 558 (Texas), and 900 (Wisconsin) arcs. Change is $100(1-J)$, where $J$ is the Jaccard index between the input arcs and the resulting arcs, and Added/Removed count arcs. For HGODE the resulting arcs are the thresholded projection of Eq.~\eqref{eq:evaluation_support_projection} at $\eta=0.5\mu$ (mean $\pm$ standard deviation over three runs), while the forward operator remains a soft gate on $\mathcal{E}_{cand}$. BORF budgets are arcs added/removed per loop; HGODE anneals its gate temperature in all four settings.}
    \label{tab:rewiring_settings}
    \centering
    {\scriptsize
        \setlength{\tabcolsep}{3pt}
        \renewcommand{\arraystretch}{1.05}
        \begin{tabular}{@{}ll>{\raggedright\arraybackslash}p{0.40\textwidth}ccc@{}}
            \toprule
            Dataset   & Method  & Selected setting                                                                         & Change (\%)     & Added               & Removed             \\
            \midrule
            CiteSeer  & BORF    & 3 loops; budget 31/15                                                                    & 2.97            & 186                 & 90                  \\
                      & GDC/PPR & $\alpha=0.15$, top-$k=16$                                                                & 82.53           & 41{,}130            & 329                 \\
                      & HGODE   & $u_{obs}/u_{new}=3/{-0.72}$; $\lambda=0.25$, $s=1.1$, $\beta=0.15$, $\delta=0.12$        & 13.89$\pm$12.78 & 1{,}621$\pm$1{,}540 & 0                   \\
            \midrule
            Chameleon & BORF    & 3 loops; budget 157/157                                                                  & 2.58            & 696                 & 942                 \\
                      & GDC/PPR & $\alpha=0.15$, top-$k=8$                                                                 & 79.90           & 2{,}770             & 49{,}573            \\
                      & HGODE   & $u_{obs}/u_{new}=0.05/{-0.3}$; $\lambda=0.4$, $s=1.5$, $\beta=0.3$, $\delta=0.2$         & 18.21$\pm$6.88  & 2{,}140$\pm$2{,}214 & 9{,}715$\pm$4{,}155 \\
            \midrule
            Texas     & BORF    & 3 loops; budget 3/0 (no removal)                                                         & 3.12           & 18                  & 0                   \\
                      & GDC/PPR & $\alpha=0.3$, top-$k=8$                                                                  & 69.30           & 849                 & 126                 \\
                      & HGODE   & $u_{obs}/u_{new}=3/{-0.3}$; $\lambda=0.25$, $s=1.1$, $\beta=0.15$, $\delta=0.12$         & 23.72$\pm$29.45 & 280$\pm$414         & 0                   \\
            \midrule
            Wisconsin & BORF    & 2 loops; budget 25/25                                                                    & 19.84           & 98                  & 100                 \\
                      & GDC/PPR & $\alpha=0.15$, top-$k=16$                                                                & 79.77           & 2{,}980             & 115                 \\
                      & HGODE   & $u_{obs}/u_{new}=3/{-0.12}$; $\lambda=0.25$, $s=1.1$, $\beta=0.15$, $\delta=0.12$        & 47.93$\pm$30.31 & 1{,}195$\pm$960     & 0                   \\
            \bottomrule
        \end{tabular}}
\end{table}

\paragraph{Findings.}
The matched-GCN controls show no monotone relation between block structure and accuracy. On Wisconsin, BORF and GDC/PPR both improve accuracy over the input graph (from 59.48\% to 63.40\% and 65.36\%), although BORF increases $\lambda_2$ and the off-block mass while GDC/PPR decreases both. On Chameleon, GDC/PPR produces the strongest near-decomposability ($\lambda_2$ from 0.00640 to 0.00165, near-zero modes from 2 to 9), yet accuracy slightly decreases (from 56.29\% to 55.77\%). On Texas, GDC/PPR lowers the off-block mass from 0.2509 to 0.0643 while accuracy is unchanged, although macro-F1 improves. HGODE does not maximize block structure either: it lowers the off-block mass on CiteSeer, Texas, and Wisconsin but adds no near-zero modes, and its $\lambda_2$ exceeds that of the input graph on Texas and Wisconsin. The learned weighted operators thus remain globally well connected while redistributing propagation weight, a pattern of selective reweighting rather than decomposition into blocks. This is consistent with Proposition~\ref{prop:blockwise_consensus}, which characterizes the invariant subspace under exact support separation and makes no claim about downstream optimality. On the task side, HGODE's accuracy is below that of the matched GCN on CiteSeer and Chameleon and above it on Texas and Wisconsin, where its macro-F1 is also roughly twice as high; because the backbones differ, these values are not a rewiring-only comparison. With three runs and small, imbalanced Texas and Wisconsin test sets (37 and 51 nodes), all differences are descriptive.

\subsection{Efficiency profiling}
\label{app:efficiency}

Times in Table~\ref{tab:efficiency} are measured on one NVIDIA A100 40GB GPU
and include the complete model-side training or inference process: candidate
lookup/gating, ODE integration, task-head computation, loss/metric computation,
and backward/optimizer steps for training. For node-classification datasets,
the reported training time is milliseconds per full-graph epoch and inference
time is milliseconds per full-graph evaluation pass. For graph-level datasets,
times are milliseconds per mini-batch; each mini-batch is packed with as many
graphs as fit into the 40GB memory budget under the corresponding model.
Forward/backward NFEs are averaged over test batches for graph-level datasets
and over the full-graph evaluation pass for node-level datasets.

\begin{table}[tb]
    \caption{Efficiency profiling on one NVIDIA A100 40GB GPU in FP32 with adjoint training and hidden dimension 256. ``--'' denotes an unavailable NFE count for the corresponding fixed-flow implementation.}
    \label{tab:efficiency}
    \centering
    {\footnotesize
        \setlength{\tabcolsep}{4pt}
        \renewcommand{\arraystretch}{1.05}
        \begin{tabular}{llccrrrr}
            \toprule
            Dataset       & Model & \multicolumn{2}{c}{NFE} & \multicolumn{2}{c}{Time [ms]} & \multicolumn{2}{c}{Memory [MB]}                              \\
            \cmidrule(lr){3-4}\cmidrule(lr){5-6}\cmidrule(lr){7-8}
                          &       & Fwd.                    & Bwd.                          & Train                           & Infer. & Train   & Infer.  \\
            \midrule
            Cora          & HGODE & 86                      & 128                           & 356.79                          & 111.11 & 548.49  & 237.39  \\
            Cora          & GRAND & 82                      & 117                           & 324.64                          & 109.45 & 509.69  & 179.18  \\
            Cora          & FLODE & --                      & --                            & 317.25                          & 106.17 & 522.16  & 234.64  \\
            Cora          & GREAD & 76                      & 124                           & 327.16                          & 98.87  & 549.34  & 201.68  \\
            \midrule
            Chameleon     & HGODE & 74                      & 110                           & 303.52                          & 85.87  & 387.47  & 163.85  \\
            Chameleon     & GRAND & 77                      & 102                           & 312.48                          & 89.44  & 331.34  & 83.16   \\
            Chameleon     & FLODE & --                      & --                            & 298.14                          & 64.38  & 378.61  & 160.82  \\
            Chameleon     & GREAD & 81                      & 106                           & 316.37                          & 68.19  & 364.34  & 146.37  \\
            \midrule
            ZINC          & HGODE & 88                      & 127                           & 2630.62                         & 649.51 & 8078.65 & 4546.13 \\
            ZINC          & GRAND & 84                      & 118                           & 2232.48                         & 526.18 & 6465.37 & 3589.16 \\
            ZINC          & FLODE & --                      & --                            & 2429.95                         & 622.81 & 7801.24 & 4678.24 \\
            ZINC          & GREAD & 87                      & 104                           & 2559.34                         & 619.13 & 7762.36 & 4518.39 \\
            \midrule
            Peptides-func & HGODE & 92                      & 134                           & 2715.87                         & 740.47 & 8051.21 & 4665.97 \\
            Peptides-func & GRAND & 87                      & 124                           & 2348.34                         & 529.16 & 6466.17 & 3674.13 \\
            Peptides-func & FLODE & --                      & --                            & 2431.64                         & 654.34 & 7762.33 & 4526.18 \\
            Peptides-func & GREAD & 78                      & 116                           & 2498.36                         & 668.36 & 7893.47 & 4329.16 \\
            \bottomrule
        \end{tabular}}
\end{table}

\subsection{Hyperparameter search and reproducibility}
\label{app:reproducibility}

We use adaptive \texttt{dopri5} with \texttt{rtol}=\texttt{atol}=$10^{-5}$ in
all reported HGODE experiments.
The optimizer is AdamW. Unless otherwise noted,
non-GCN baselines are matched to roughly comparable parameter budgets by
adjusting hidden dimensions. Candidate sets are constructed before integration
by Algorithm~\ref{alg:candidate_pool}.
training labels, as described in Sec.~\ref{sec:margin_objective}. For
Chameleon, we follow Platonov et al.~\cite{platonov2022critical}: the filtered
variant removes duplicate nodes from the standard splits to avoid train--test
leakage. For Roman-empire, we follow their ten fixed random 50/25/25
train/validation/test splits and report accuracy.

Table~\ref{tab:hparam_roles} groups the HGODE hyperparameters by role.

\begin{table}[tb]
    \caption{Roles of the HGODE hyperparameters; values are given in Tables~\ref{tab:selected_hparams}--\ref{tab:hparam_starts}.}
    \label{tab:hparam_roles}
    \centering
    {\footnotesize
        \setlength{\tabcolsep}{4pt}
        \renewcommand{\arraystretch}{1.1}
        \begin{tabular}{@{}>{\raggedright\arraybackslash}p{0.24\textwidth}>{\raggedright\arraybackslash}p{0.46\textwidth}>{\raggedright\arraybackslash}p{0.25\textwidth}@{}}
            \toprule
            Role                                  & Hyperparameters                                                                                                                                                                       & Rule                                                                                \\
            \midrule
            Implementation constants              & two message blocks; force-MLP widths $(64,32)$; $\epsilon_{\mathrm{loop}}=10^{-3}$, $\epsilon_{\mathrm{deg}}=10^{-6}$; \texttt{rtol}$=$\texttt{atol}$=10^{-5}$; gradient clipping $1.0$ & Not tuned                                                                           \\
            Shared by all selected configurations & hidden width $d$, dropout $p$, $\tau_{feat}$, $\tau_{topo}$                                                                                                                          & Searched (Table~\ref{tab:hparam_ranges}); identical in Table~\ref{tab:selected_hparams} \\
            Constrained by the theory             & $0\le\lambda<1$; $\tau_{feat},\tau_{topo},T>0$; $s\ge\mathcal{F}_{\mathrm{crit}}(\lambda)+\delta$ whenever $\beta>0$                                                                & Required for a well-posed, feasible margin target                                   \\
            Validation-selected                   & $\lambda$, gate temperature $\tau$, $\gamma$, $s$, $\delta$, $\beta$, $\rho_{ml}$, $T$, learning rate                                                                                & Tables~\ref{tab:hparam_ranges} and~\ref{tab:hparam_starts}                          \\
            Candidate budget                      & $k_{2hop}$, $k_{lap}$, $r_{rand}$, together with the size caps of Algorithm~\ref{alg:candidate_pool}                                                                                & Selected on validation under a resource budget                                      \\
            \bottomrule
        \end{tabular}}
\end{table}

Table~\ref{tab:selected_hparams} is the compact reporting table for final
selected configurations. All experiments are rerun 5 times and the seeds are
all chosen as 42, 1042, 2042, 3042 and 4042. In addition, we attach the general
hyperparameter tuning range in Table~\ref{tab:hparam_ranges} and the
recommended hyperparameter tuning range for each dataset in
Table~\ref{tab:hparam_starts}. Dataset provenance, benchmark splits, and
available license or terms information are summarized in
Table~\ref{tab:dataset_assets}.

In the hyperparameters, $k_{2hop}$ and $k_{lap}$ are the number of 2-hop and
Laplacian-based structural candidates per node, respectively. $r_{rand}$ is the
random candidate ratio per node. These three hyperparameters determine the
candidate-pool size (Algorithm~\ref{alg:candidate_pool}). In our chosen hyperparameter combinations, those values
are chosen to reach good-enough performance without any additional
computational cost, \ie, without need to increase the candidate-pool size. We
run all the models with the same epochs: 80 epochs on Cora, Chameleon; 200 for
the other datasets. The early-stopping patience is 10 epochs for all the
training. We select the checkpoint and hyperparameters by validation performance and report the corresponding test performance. On the single-graph benchmarks, we run the models for a single graph per
epoch, and for the multi-graph benchmarks, we load each batch to fulfill our
experiment GPU memory budget (40 GB on single A100 GPU). Gradient clips are set
to 1.0 for all the training. All the benchmarks results are reported with our
rerun results.

\begin{table}[tb]
    \caption{Hyperparameter chosen for each reported benchmarks.}
    \label{tab:selected_hparams}
    \centering
    {\scriptsize
        \setlength{\tabcolsep}{2.2pt}
        \renewcommand{\arraystretch}{1.05}
        \resizebox{\textwidth}{!}{%
            \begin{tabular}{lcccccccccccccccccc}
                \toprule
                Dataset       & Split              & $\lambda$ & $\tau$ & $\tau_{feat}$ & $\tau_{topo}$ & $\gamma$ & $s$ & $\delta$ & $\beta$ & $T$ & hid. & drop. & lr   & $k_{2hop}$ & $k_{lap}$ & $r_{rand}$ \\
                \midrule
                Cora          & Planetoid          & 0.3       & 0.2    & 1.0      & 1.0      & 0.05     & 1.0 & 0.1      & 0.3     & 0.3 & 256  & 0.5   & 5e-4 & 4          & 6         & 0.1        \\
                Chameleon     & Chameleon-filtered & 0.4       & 0.05   & 1.0      & 1.0      & 0.05     & 1.5 & 0.2      & 0.3     & 0.8 & 256  & 0.5   & 5e-4 & 4          & 6         & 0.1        \\
                Roman-empire  & Platonov 50/25/25  & 0.5       & 0.3    & 1.0      & 1.0      & 0.05     & 1.0 & 0.2      & 0.3     & 1.0 & 256  & 0.5   & 5e-4 & 4          & 6         & 0.1        \\
                ogbn-proteins & OGB official       & 0.5       & 0.1    & 1.0      & 1.0      & 0.1      & 1.0 & 0.1      & 0.1     & 1.0 & 256  & 0.5   & 1e-4 & 6          & 12        & 1e-3       \\
                ZINC          & PyG                & 0.5       & 0.1    & 1.0      & 1.0      & 0.05     & 1.0 & 0.1      & 0.2     & 1.0 & 256  & 0.5   & 5e-4 & 4          & 6         & 0.01       \\
                Peptides-func & LRGB               & 0.6       & 0.2    & 1.0      & 1.0      & 0.1      & 1.0 & 0.1      & 0.1     & 1.5 & 256  & 0.5   & 1e-4 & 6          & 12        & 0.05       \\
                ogbg-molpcba  & OGB official       & 0.5       & 0.1    & 1.0      & 1.0      & 0.05     & 1.0 & 0.2      & 0.1     & 1.0 & 256  & 0.5   & 1e-4 & 4          & 6         & 0.1        \\
                \bottomrule
            \end{tabular}}}
\end{table}

\begin{table}[tb]
    \caption{HGODE hyperparameter search space used for the reported experiments.}
    \label{tab:hparam_ranges}
    \centering
    {\footnotesize
        \setlength{\tabcolsep}{4pt}
        \renewcommand{\arraystretch}{1.05}
        \begin{tabular}{p{0.22\textwidth}p{0.22\textwidth}p{0.48\textwidth}}
            \toprule
            Group                   & Hyperparameter            & Range                                            \\
            \midrule
            Structural / hysteresis & $\lambda$                 & $\{0.1,0.3,0.4,0.5,0.6,0.8\}$                    \\
            Structural / hysteresis & $\tau$                    & $\{0.05,0.1,0.2,0.3,0.5\}$                       \\
            Structural / hysteresis & $\tau_{feat},\tau_{topo}$ & $\{0.3,0.5,1.0\}$                                \\
            Structural / hysteresis & $\gamma$                  & $\{0,0.05,0.1,0.5\}$                             \\
            Force / margin          & $s$                       & $\{1,1.5\}$                                      \\
            Force / margin          & $\delta$                  & $\{0.1,0.2,0.3,0.5\}$                            \\
            Force / margin          & $\beta$                   & $\{0,0.1,0.2,0.3,0.5\}$                          \\
            Force / margin          & $\rho_{ml}$               & $\{0.4,0.5\}$                                    \\
            Candidate pool          & $r_{rand}$                & $\{0,10^{-3},10^{-2},5{\times}10^{-2},10^{-1}\}$ \\
            Candidate pool          & $k_{2hop}$                & $\{0,4, 6, 8,12\}$                               \\
            Candidate pool          & $k_{lap}$                 & $\{0,2,4,6,8,10,12\}$                            \\
            Backbone / optimization & \texttt{hidden\_dim}      & $\{128,256,512\}$                                \\
            Backbone / optimization & \texttt{dropout}          & $\{0.2,0.5\}$                                    \\
            Backbone / optimization & \texttt{lr}               & $\{10^{-4},5{\times}10^{-4},10^{-3}\}$           \\
            Backbone / optimization & $T$                       & $\{0.3,0.6,0.8,1.0,1.5,2.0\}$                    \\
            \bottomrule
        \end{tabular}}
\end{table}

\begin{table}[tb]
    \caption{Representative HGODE starting subgrids by dataset regime. These are not additional search values; each entry is a subset of Table~\ref{tab:hparam_ranges}.}
    \label{tab:hparam_starts}
    \centering
    {\footnotesize
        \setlength{\tabcolsep}{4pt}
        \renewcommand{\arraystretch}{1.25}
        \begin{tabular}{@{}p{0.15\textwidth}p{0.20\textwidth}p{0.35\textwidth}p{0.23\textwidth}@{}}
            \toprule
            Dataset regime       & Examples                     & Hysteresis subgrid                                                                                    & Force subgrid                                                      \\
            \midrule
            Homophilous local    & Cora, ZINC                   & $\lambda\in\{0.1,0.3\}$;\ $\tau\in\{0.2,0.3\}$;\ $\tau_{feat}=\tau_{topo}=1.0$                        & $s=1.0$;\ $\delta=0.1$;\ $\beta\in\{0,0.1\}$                       \\
            Homophilous global   & ogbg-molpcba                 & $\lambda\in\{0.3,0.5\}$;\ $\tau\in\{0.1,0.2\}$;\ $\tau_{feat}=\tau_{topo}=1.0$                        & $s\in\{1.0,1.5\}$;\ $\delta\in\{0.1,0.2\}$;\ $\beta\in\{0.1,0.3\}$ \\
            Heterophilous sparse & Chameleon, Roman-empire      & $\lambda\in\{0.4,0.5,0.6\}$;\ $\tau\in\{0.05,0.1\}$;\ $\tau_{feat}=\tau_{topo}=1.0$                   & $s\in\{1.0,1.5\}$;\ $\delta\in\{0.2,0.3\}$;\ $\beta\in\{0.3,0.5\}$ \\
            Heterophilous global & ogbn-proteins, Peptides-func & $\lambda\in\{0.5,0.6,0.8\}$;\ $\tau\in\{0.05,0.1\}$;\ $\tau_{feat}=1.0$;\ $\tau_{topo}\in\{0.5,1.0\}$ & $s\in\{1.0,1.5\}$;\ $\delta\in\{0.2,0.3\}$;\ $\beta\in\{0.1,0.3\}$ \\
            \bottomrule
        \end{tabular}}
\end{table}

\begin{table}[tb]
    \caption{Dataset provenance, splits, and license or terms information. When an upstream benchmark wrapper or repository does not state a separate dataset license, we report that explicitly rather than inferring one.}
    \label{tab:dataset_assets}
    \centering
    {\footnotesize
        \setlength{\tabcolsep}{3pt}
        \renewcommand{\arraystretch}{1.12}
        \begin{tabular}{p{0.13\textwidth}p{0.28\textwidth}p{0.15\textwidth}p{0.36\textwidth}}
            \toprule
            Asset         & Source / access path                                                                                                                                               & Split used                  & License / terms reported by source                                                                                                                \\
            \midrule
            Cora          & Planetoid / PyG citation-network wrapper; original Planetoid benchmark~\cite{yang2016revisiting}                                                                   & Planetoid public split      & No separate dataset license stated in the PyG wrapper; used as a public research benchmark with citation to the original source.                  \\
            Chameleon     & WikipediaNetwork / PyG wrapper with duplicate filtering from Platonov et al.~\cite{platonov2022critical}; original Wikipedia networks~\cite{rozemberczki2021multi} & Chameleon-filtered splits   & Filtered benchmark and code are released in the Platonov et al. repository under MIT; original Wikipedia network source is cited.                 \\
            Roman-empire  & Platonov et al. heterophilous-graphs benchmark / PyG wrapper~\cite{platonov2022critical}                                                                           & 10 random 50/25/25 splits   & Platonov et al. repository is MIT-licensed; dataset is based on the English Wikipedia Roman Empire article with citation to the source benchmark. \\
            ogbn-proteins & Open Graph Benchmark~\cite{hu2020open}                                                                                                                             & OGB official species split  & CC-0 according to the OGB node-property benchmark page.                                                                                           \\
            ZINC          & PyG ZINC wrapper for the Benchmarking GNNs subset~\cite{JMLR:v24:22-0567}                                                                                          & Standard split              & ZINC is distributed as a free molecular database; no separate PyG benchmark-subset license is stated.                                             \\
            Peptides-func & Long Range Graph Benchmark~\cite{dwivedi2022long}                                                                                                                  & LRGB official split         & CC-BY-NC-4.0 according to the LRGB Hugging Face dataset card.                                                                                     \\
            ogbg-molpcba  & Open Graph Benchmark~\cite{hu2020open}                                                                                                                             & OGB official scaffold split & MIT according to the OGB graph-property benchmark page.                                                                                           \\
            \midrule
            CiteSeer (App.~\ref{app:topology_diagnostics}) & Planetoid / PyG citation-network wrapper; original Planetoid benchmark~\cite{yang2016revisiting} & Planetoid public split & No separate dataset license stated in the PyG wrapper; used as a public research benchmark with citation to the original source. \\
            Chameleon, Texas, Wisconsin (App.~\ref{app:topology_diagnostics}) & WikipediaNetwork and WebKB / PyG wrappers with Geom-GCN splits~\cite{pei2020geom}; original Wikipedia networks~\cite{rozemberczki2021multi} & Geom-GCN split columns 0--2 & No separate dataset license stated in the PyG wrappers; used as public research benchmarks with citation to the original sources. \\
            \bottomrule
        \end{tabular}}
\end{table}

\section{Bistability and Candidate-Pool Interpretation}
\label{app:bistability_candidate_pool}

Per-edge bistability is sufficient for multi-class structure because the local
decision is binary: should this pair support feature mixing, or should it be
insulated? A graph with many candidate edges then has many possible stable sign
patterns, so global multistability arises compositionally from many locally
bistable edge units. When intra-cluster edges polarize to the connected phase
and cross-cluster edges polarize to the insulated phase, the effective
propagation matrix becomes approximately block-structured even though each
individual edge has only two phases.

The candidate pool is a scalability design rather than a separate modeling
assumption. The full all-pairs version evolves $\mathbf{U}(t)\in\mathbb
    R^{N\times N}$ and is conceptually valid when resources permit. Sparse pools
choose a tractable subset of this full matrix: observed edges preserve the
input graph support, local proposals add multi-hop completion,
Laplacian-eigenvector proposals add global relations, and random pairs provide exploration.
Increasing the pool size therefore interpolates between graph-support
polarization and dense topology search. In the notation of an
implementation-level neighbor ratio $r$, the dense case is $r=1$, where every
ordered pair is included and HGODE reduces to dense-matrix topology evolution.

\subsection{Coverage from small random proposal ratios}
\label{app:candidate_pool_coverage}

The purpose of a small random-pair ratio is not to include every possible
long-distance edge, but to ensure that each relation group with many possible
node pairs is likely to be represented at least once. The following elementary
bound makes this statement explicit.

\begin{proposition}[Hitting relation groups with sparse proposals]
    Let $\mathcal{R}_1,\dots,\mathcal{R}_K\subseteq \mathcal{V}\times\mathcal{V}$ be relation groups of interest, such as long-distance cross-cluster pairs, and let $m_k=|\mathcal{R}_k|$. Suppose each ordered pair is independently added to the random proposal pool with probability $r\in[0,1]$. Then
    \[
        \Pr\big[\forall k,\; \mathcal{E}_{cand}\cap \mathcal{R}_k\neq\varnothing\big]
        \ge 1-\sum_{k=1}^{K}(1-r)^{m_k}.
    \]
    In particular, if $m_{\min}=\min_k m_k>0$ and $\eta\in(0,1)$, then $r\ge
        \frac{1}{m_{\min}}\log\frac{K}{\eta}$ is sufficient to cover every group with
    probability at least $1-\eta$.
\end{proposition}

\begin{proof}
    A relation group $\mathcal{R}_k$ is missed only if none of its $m_k$ ordered pairs is sampled, which has probability $(1-r)^{m_k}$. A union bound over $K$ groups gives the first inequality. Since $(1-r)^{m_k}\le e^{-rm_k}\le e^{-r m_{\min}}$, choosing $r\ge m_{\min}^{-1}\log(K/\eta)$ makes the failure probability at most $\eta$.
\end{proof}

This bound explains why small ratios can cover long-distance relation
\emph{types} when each type contains many possible node pairs, as is common in
large graphs. 2-hop and Laplacian-eigenvector proposals provide short- and mid-range
coverage, while the random component gives nonzero probability to far-apart
pairs that would otherwise be absent from a local candidate pool.

\newpage
\section*{NeurIPS Paper Checklist}


\begin{enumerate}

    \item {\bf Claims}
    \item[] Question: Do the main claims made in the abstract and introduction accurately reflect the paper's contributions and scope?
    \item[] Answer: \answerYes{}
    \item[] Justification: The title, abstract, and introduction claim mitigation of global consensus collapse and state the positivity assumptions; Table~\ref{tab:theory_regimes} separates asymptotic global consensus, idealized block-wise consensus, and finite-horizon mitigation by the implemented soft gates, matching Secs.~\ref{sec:monostable_trap}--\ref{sec:realworld}.

    \item {\bf Limitations}
    \item[] Question: Does the paper discuss the limitations of the work performed by the authors?
    \item[] Answer: \answerYes{}
    \item[] Justification: Sec.~\ref{sec:limitations} discusses the scope of the consensus theory (consensus-form propagation and no global convergence guarantee for the learned coupled system), the evaluated operator family, directed-graph coverage, candidate-pool cost, and graph-structure-learning risks; Table~\ref{tab:theory_regimes} states the finite-horizon scope of the soft-gate result.

    \item {\bf Theory assumptions and proofs}
    \item[] Question: For each theoretical result, does the paper provide the full set of assumptions and a complete (and correct) proof?
    \item[] Answer: \answerYes{}
    \item[] Justification: Theoretical assumptions are stated with each theorem/lemma in Sec.~\ref{sec:monostable_trap}, and the corresponding proofs appear in Appendix~\ref{sec:proofs}.

    \item {\bf Experimental result reproducibility}
    \item[] Question: Does the paper fully disclose all the information needed to reproduce the main experimental results of the paper to the extent that it affects the main claims and/or conclusions of the paper (regardless of whether the code and data are provided or not)?
    \item[] Answer: \answerYes{}
    \item[] Justification: The appendix now gives architecture, algorithm, solver, search ranges, efficiency protocol, split/provenance fields, and the selected-configuration reporting table in Table~\ref{tab:selected_hparams}.

    \item {\bf Open access to data and code}
    \item[] Question: Does the paper provide open access to the data and code, with sufficient instructions to faithfully reproduce the main experimental results, as described in supplemental material?
    \item[] Answer: \answerNo{}
    \item[] Justification: We use public benchmark datasets and provide algorithmic, solver, hyperparameter, split, and reporting details in Appendix~\ref{app:supplementary_details}. 

    \item {\bf Experimental setting/details}
    \item[] Question: Does the paper specify all the training and test details (e.g., data splits, hyperparameters, how they were chosen, type of optimizer) necessary to understand the results?
    \item[] Answer: \answerYes{}
    \item[] Justification: The paper reports optimizer, solver tolerances, benchmark splits, architecture, search ranges, selected hyperparameters, seeds, candidate-pool settings, and efficiency protocol in Appendix~\ref{app:reproducibility}.

    \item {\bf Experiment statistical significance}
    \item[] Question: Does the paper report error bars suitably and correctly defined or other appropriate information about the statistical significance of the experiments?
    \item[] Answer: \answerYes{}
    \item[] Justification: Main benchmark tables report mean $\pm$ standard deviation, and the synthetic perturbation study states that shaded regions are $\pm 1$ standard deviation over 10 seeds.

    \item {\bf Experiments compute resources}
    \item[] Question: For each experiment, does the paper provide sufficient information on the computer resources (type of compute workers, memory, time of execution) needed to reproduce the experiments?
    \item[] Answer: \answerYes{}
    \item[] Justification: Appendix~\ref{app:efficiency} reports the A100 40GB setting, NFE, time, memory, and the measurement unit for node- and graph-level experiments.

    \item {\bf Code of ethics}
    \item[] Question: Does the research conducted in the paper conform, in every respect, with the NeurIPS Code of Ethics \url{https://neurips.cc/public/EthicsGuidelines}?
    \item[] Answer: \answerYes{}
    \item[] Justification: The work uses standard public graph benchmarks and method development; Sec.~\ref{sec:limitations} discusses responsible use of learned graph relations.

    \item {\bf Broader impacts}
    \item[] Question: Does the paper discuss both potential positive societal impacts and negative societal impacts of the work performed?
    \item[] Answer: \answerYes{}
    \item[] Justification: Sec.~\ref{sec:limitations} discusses potential benefits for graph learning and risks from learned or inferred graph relations, including bias amplification and causal misinterpretation.

    \item {\bf Safeguards}
    \item[] Question: Does the paper describe safeguards that have been put in place for responsible release of data or models that have a high risk for misuse (e.g., pre-trained language models, image generators, or scraped datasets)?
    \item[] Answer: \answerNA{}
    \item[] Justification: The paper does not release a high-risk foundation model, scraped dataset, or other asset requiring special misuse safeguards.

    \item {\bf Licenses for existing assets}
    \item[] Question: Are the creators or original owners of assets (e.g., code, data, models), used in the paper, properly credited and are the license and terms of use explicitly mentioned and properly respected?
    \item[] Answer: \answerYes{}
    \item[] Justification: Table~\ref{tab:dataset_assets} lists the benchmark sources, splits, and available license or terms information; when an upstream wrapper does not specify a separate dataset license, the table states this explicitly and the paper uses the public benchmark split with citation.

    \item {\bf New assets}
    \item[] Question: Are new assets introduced in the paper well documented and is the documentation provided alongside the assets?
    \item[] Answer: \answerNA{}
    \item[] Justification: The paper introduces a method and experimental implementation, but the current manuscript does not release a new dataset or standalone asset requiring separate documentation.

    \item {\bf Crowdsourcing and research with human subjects}
    \item[] Question: For crowdsourcing experiments and research with human subjects, does the paper include the full text of instructions given to participants and screenshots, if applicable, as well as details about compensation (if any)?
    \item[] Answer: \answerNA{}
    \item[] Justification: The experiments do not involve crowdsourcing or human subjects.

    \item {\bf Institutional review board (IRB) approvals or equivalent for research with human subjects}
    \item[] Question: Does the paper describe potential risks incurred by study participants, whether such risks were disclosed to the subjects, and whether Institutional Review Board (IRB) approvals (or an equivalent approval/review based on the requirements of your country or institution) were obtained?
    \item[] Answer: \answerNA{}
    \item[] Justification: The experiments do not involve crowdsourcing or human subjects, so IRB approval is not applicable.

    \item {\bf Declaration of LLM usage}
    \item[] Question: Does the paper describe the usage of LLMs if it is an important, original, or non-standard component of the core methods in this research? Note that if the LLM is used only for writing, editing, or formatting purposes and does \emph{not} impact the core methodology, scientific rigor, or originality of the research, declaration is not required.
    \item[] Answer: \answerNA{}
    \item[] Justification: LLMs are not part of the core HGODE methodology, experiments, or scientific contribution described in the paper.

\end{enumerate}

\end{document}